\documentclass[twoside,11pt,letter]{article}

\usepackage{epsfig,graphics,latexsym}
\usepackage{amsmath,amssymb,theorem,enumerate}
\usepackage{jair}
\usepackage{theapa}
\usepackage{array}
\usepackage{wasysym}
\usepackage{multirow}
\usepackage{stmaryrd}
\usepackage{float}

\usepackage{graphicx}

\input xy
\xyoption{all}

\def\defemb#1#2{\expandafter\def\csname #1\endcsname
                              {\relax\ifmmode #2\else\hbox{$#2$}\fi}}

\defemb{bp}{{\bf p}}
\defemb{bU}{{\bf U}}
\defemb{bu}{{\bf u}}
\defemb{bV}{{\bf V}}
\defemb{bW}{{\bf W}}
\defemb{bw}{{\bf w}}
\defemb{bX}{{\bf X}}
\defemb{bx}{{\bf x}}
\defemb{bY}{{\bf Y}}
\defemb{by}{{\bf y}}
\defemb{bZ}{{\bf Z}}
\defemb{bz}{{\bf z}}

\defemb{cA}{{\cal A}}
\defemb{cB}{{\cal B}}
\defemb{cC}{{\cal C}}
\defemb{cc}{{\cal c}}
\defemb{cD}{{\cal D}}
\defemb{cE}{{\cal E}}
\defemb{cF}{{\cal F}}
\defemb{cf}{{\cal f}}

\defemb{cG}{{\cal G}}
\defemb{GG}{{\cal G}}

\defemb{cI}{{\cal I}}
\defemb{cK}{{\cal K}}
\defemb{cP}{{\cal P}}
\defemb{cL}{{\cal L}}
\defemb{cN}{{\cal N}}
\defemb{cM}{{\cal M}}
\defemb{cR}{{\cal R}}

\defemb{cS}{{\cal S}}
\defemb{SS}{{\cal S}}

\defemb{cT}{{\cal T}}
\defemb{cV}{{\cal V}}
\defemb{cU}{{\cal U}}
\defemb{cO}{{\cal O}}
\defemb{cH}{{\cal H}}
\defemb{cX}{{\cal X}}
\defemb{cY}{{\cal Y}}
\defemb{cZ}{{\cal Z}}

\defemb{sP}{{\sf P}}
\defemb{sG}{{\sf G}}
\defemb{sU}{{\sf U}}
\defemb{sE}{{\sf E}}
\defemb{sI}{{\sf I}}
\defemb{su}{{\sf u}}
\defemb{se}{{\sf e}}
\defemb{ss}{{\sf s}}
\defemb{st}{{\sf t}}
\defemb{sp}{{\sf p}}
\defemb{si}{{\sf i}}
\defemb{sb}{{\sf b}}
\defemb{sw}{{\sf w}}
\defemb{sA}{{\sf A}}
\defemb{sB}{{\sf B}}
\defemb{sa}{{\sf a}}
\defemb{sb}{{\sf b}}
\defemb{scc}{{\sf c}}
\defemb{sd}{{\sf d}}

\defemb{tta}{{\tt a}}
\defemb{ttb}{{\tt b}}
\defemb{ttc}{{\tt c}}
\defemb{ttd}{{\tt d}}
\defemb{tte}{{\tt e}}
\defemb{ttf}{{\tt f}}

\defemb{bara}{{\bar{a}}}
\defemb{barb}{{\bar{b}}}
\defemb{barc}{{\bar{c}}}
\defemb{bard}{{\bar{d}}}
\defemb{barf}{{\bar{f}}}
\defemb{barv}{{\bar{v}}}
\defemb{barx}{{\bar{x}}}

\newcommand{\commentout}[1]{}

\newtheorem{definition}{Definition}
{\theorembodyfont{\normalfont}

\newtheorem{example}{Example}
}

{\theorembodyfont{\bf \sf}
\newtheorem{theorem}{Theorem}
\newtheorem{lemma}{Lemma}

\newtheorem{corollary}{Corollary}}

\newenvironment{proof}{{\vspace{0.5cm}\noindent \bf Proof:}~~}
                      {$\square$\vspace{0.3cm}}
\newenvironment{nproof}[1]{{\vspace{0.5cm}\noindent\bf Proof (#1)}~~}
                          {$\square$\vspace{0.5cm}}

\newenvironment{atheorem}[1]{{\vspace{0.4cm}\noindent\bf
Theorem~\ref{#1}}~~}
                            {}

\newcommand{\searchtcp}{{\sf Search-TCP}}

\newcommand{\CIT}{CIT}
\newcommand{\cmplet}{Comp}

\newcommand{\imp}{\rhd}

\newcommand{\TCPn}{TCP-net}
\newcommand{\tcpn}{TCP-net}
\newcommand{\TCPns}{TCP-nets}
\newcommand{\cNN}{\cN^{\star}}
\newcommand{\cparc}[2]{\langle \overrightarrow{#1, #2} \rangle}
\newcommand{\iarc}[2]{(\overrightarrow{#1, #2})}
\newcommand{\ciarc}[2]{(#1, #2)}

\newcommand{\shared}{{\mathsf{shared}}}

\jairheading{25}{2006}{389--424}{09/05}{03/06}
\ShortHeadings{TCP-Nets}
{Brafman, Domshlak, \& Shimony}
\firstpageno{389}

\begin{document}


\title{On Graphical Modeling of Preference and Importance}

\author{\name Ronen I. Brafman \email brafman@cs.stanford.edu \\
\addr Department of Computer Science \\
     Stanford University \\
      Stanford CA 94305
\AND
\name Carmel Domshlak \email dcarmel@ie.technion.ac.il \\
\addr Faculty of Industrial Engineering and Management \\
      Technion - Israel Institute of Technology\\
      Haifa, Israel 32000
\AND
\name Solomon E. Shimony \email shimony@cs.bgu.ac.il \\
\addr Department of Computer Science \\
     Ben-Gurion University \\
      Beer Sheva, Israel 84105
}


\maketitle

\begin{abstract}
\commentout{
  The ability to make decisions and to assess potential courses of
  action is a corner-stone of many AI applications. Typically,
  this ability requires explicit information about the decision-maker's
  preferences.  In many applications, preference elicitation is a
  serious bottleneck.  The user either does not have the time, the
  knowledge, or the expert support required to specify complex
  multi-attribute utility functions. In such cases, a method that is
  based on intuitive, yet expressive, preference statements is
  required.  In this paper we suggest the use of TCP-nets, an
  enhancement of CP-nets, as a tool for representing, and reasoning
  about qualitative preference statements. We present and motivate
  this framework, define its semantics, and study various computational
  aspects of reasoning with TCP-nets. Finally, we show how to
  perform constrained optimization efficiently given a TCP-net.
  }
  In recent years, CP-nets have emerged as a useful tool for supporting preference
  elicitation, reasoning, and representation. CP-nets capture and support reasoning with
  qualitative conditional preference statements, statements that are relatively natural for users to express.
  In this paper, we extend the CP-nets formalism to handle another
  class of very natural qualitative statements one
  often uses in expressing preferences in daily life -- statements
  of relative importance of attributes. The resulting formalism, TCP-nets, maintains
  the spirit of CP-nets, in that it remains focused on using only
  simple and natural preference statements, uses the {\em ceteris paribus\/} semantics,
  and utilizes a graphical representation of this information to reason about
  its consistency and to perform, possibly constrained, optimization using
  it. The extra expressiveness it provides allows us to
  better model tradeoffs users would like to make, more faithfully 
 representing their preferences.
  \end{abstract}

\section{Introduction}
The ability to make decisions and to assess potential courses of
action is a corner-stone of numerous AI applications, including expert
systems, autonomous agents, decision-support systems, recommender
systems, configuration software, and constrained optimization
applications. To make good decisions, we must be able to assess and
compare different alternatives. Sometimes, this comparison is
performed implicitly, as in many recommender systems~\cite{Burke:00,acmc40}.  But frequently,
explicit information about the decision-maker's preferences is
required.

In classical decision theory and decision analysis utility functions
are used to represent the decision-maker's preferences. However, the process of obtaining the type of
information required to generate a good utility function is involved, time-consuming and requires non-negligible effort on the part of
the user~\cite{french}. Sometimes such effort is necessary and possible, but in many applications
the user cannot be engaged for a lengthy period of time and cannot be
supported by a human decision analyst. For instance, this is the case in on-line product recommendation systems and other software
decision-support applications.

When a utility function cannot be or need not be obtained,
one should resort to other, more qualitative forms of preference
representation.  Ideally, this qualitative information should be
easily obtainable from the user by non-intrusive means. That is, we
should be able to extract such information from natural and relatively simple
statements of preference provided by the user, and this
elicitation process should be amenable to automation. In addition,
automated reasoning about such qualitative preference information should be semantically effective and computationally efficient.

One framework for preference representation that addresses these
concerns is that of {\em Conditional Preference Networks\/} (CP-nets)
(Boutilier et al.~\citeyearR{BBHP.UAI99,BBDHP.journal}).
CP-nets is a graphical preference representation model grounded in the notion
of conditional preferential independence.
In  preference elicitation with CP-nets, the decision maker (directly or indirectly) describes how her preference
over the values of one variable depends on the value of other
variables. For example, she may state that her preference for a
dessert depends on  the main-course as well as whether or
not she had an alcoholic beverage. In turn, her preference for an alcoholic beverage
may depend on the main course and the time of day. This information is
described by a graphical structure in which the nodes represent
variables of interest and edges capture direct preferential dependence relations between the variables.
Each node is annotated with a {\em conditional
preference table\/} (CPT) describing the user's preference over
alternative values of this node given different values of the parent
nodes.  CP-nets capture a class of intuitive and useful natural
language statements of the form ``I prefer the value $x_0$ for
variable $X$ given that $Y=y_0$ and $Z=z_0$''. Such statements do not
require complex introspection nor a quantitative assessment.

From the practical perspective,
there is another class of preference statements
that is no less intuitive or important, yet is not captured by the CP-net model. These statements have the form: ``It is more important to me that the value of $X$ be better than that the value of $Y$ be better.''
We call these {\em relative importance\/}
statements.  For instance, one might say ``The length of the journey is more important to me than the choice of airline''.  A more refined notion of
importance, though still intuitive and easy to communicate, is that of {\em
  conditional relative importance\/}: ``The length of the journey is more
important to me than the choice of airline if I need to give a
talk the following day. Otherwise, the choice of airline is more
important.'' The latter statement is of the form: ``A better
assignment for $X$ is more important than a better assignment for
$Y$ given that $Z=z_0$.''  Notice that information about relative
importance is different from information about preferential
independence. For instance, in the example above, user's
preference for an airline does not depend on the duration of the
journey because, e.g., she compares airlines based only on their
service, security levels, and the quality of their frequent flyer
program.  Informally, using statements of relative importance the
user expresses her preference over {\em compromises} that may be
required. Such information is very important in customized
product configuration
applications~\cite{Sabin:Weigel:ieee98,Haag:ieee98,Freuder:OSullivan:cp01},
where production, supply, and other constraints are posed on the
product space by the producer, and these constraints are
typically even unknown to the customer. Indeed, in many
applications, various resource (e.g., money, time, bandwidth)
constraints exist, and the main computational task is that of
finding a solution that is feasible and not preferentially dominated by any other solution.

In this paper we consider enhancing the expressive power of CP-nets by
introducing information about importance relations, obtaining a
preference-representation structure which we call TCP-nets (for {\em
tradeoffs-enhanced\/} CP-nets).
By capturing information about both conditional preferential independence and conditional relative
importance, TCP-nets provide a richer framework for representing
user preferences, allowing stronger conclusions to be drawn, yet remaining committed to the use of only very intuitive, qualitative information.  At the same time, we show that the added relative importance information has significant impact on both the consistency of the specified relation, and the techniques used for reasoning about it.  Focusing on these computational issues, we show that the graphical structure of the ``mixed'' set of preference statements captured in a TCP-net can often be exploited in order to achieve efficiency both in consistency testing and in preferential reasoning.

This paper is organized as follows: Section~\ref{S:backg}
describes the notions underlying \TCPns: preference relations,
preferential independence, and relative importance.  In
Section~\ref{S:tcpnet} we define \TCPns, specify their semantics, and provide a number of
examples.
In Section~\ref{S:acyclic} we characterize a class of conditionally acyclic TCP-nets whose consistency is guaranteed and then, in Section~\ref{S:tcp-consistency} we discuss the complexity of identifying
members of this class.
In Section~\ref{S:constraints} we present an algorithm for outcome
optimization in conditionally acyclic TCP-nets, and discuss the
related tasks of reasoning about preferences given a TCP-net. We conclude with a discussion of related and future work in
Section~\ref{S:summary}.

\section{Preference Orders, Independence, and Relative Importance}
\label{S:backg}
In this section we describe the semantic concepts underlying TCP-nets:
preference orders, preferential independence, conditional preferential
independence, as well as relative importance and conditional relative
importance.

\subsection{Preference and Independence}

We model a {\em preference relation\/} as a strict partial order.
Thus, we use the terms preference order and strict partial order
interchangeably. A strict partial order is a binary relation over
outcomes that is anti-reflexive, anti-symmetric and transitive. Given two
outcomes $o,o'$, we write $o\succ o'$ to denote that $o$ is
strictly preferred to $o'$.

The above choice implies that two outcomes cannot be equally
preferred. This choice follows from the fact that the language of
preferences we use in this paper does not allow statements of
indifference (as opposed to incomparability), and thus there is
no need for using weak orderings. Incorporating statements of
indifference is pretty straightforward, as explained by~\citeA{BBDHP.journal}, but introduces much overhead if we were
to formally treat it throughout this paper.

The types of outcomes we are concerned with consist of possible
assignments to some set of variables. More formally, we assume
some given set $\bV=\{X_1,\ldots,X_n\}$ of variables with
corresponding domains $\cD(X_1),\ldots,\cD(X_n)$. The set of
possible outcomes is then $\cD(\bV) =
\cD(X_1)\times\cdots\times\cD(X_n)$, where we use $\cD(\cdot)$ to
denote the domain of a set of variables as well. For example, in
the context of the problem of configuring a personal computer
(PC), the variables may be {\em processor type, screen size,
operating system\/} etc., where {\em screen size} has the domain
{\em \{17in, 19in, 21in\}}, {\em operating system\/} has the
domain {\em \{LINUX, Windows98, WindowsXP\}\/}, etc. Each
complete assignment to the set of variables specifies an outcome
-- a particular PC configuration. Thus, a preference relation over
these outcomes specifies a strict partial order over possible PC
configurations.

The number of possible outcomes is exponential in $n$, while the
set of possible orderings on them is more than doubly exponential
in $n$. Therefore, explicit specification and representation of
an ordering is not realistic, and thus we must describe it
implicitly using a compact representation model. The notion of
preferential independence plays a key role in such
representations. Intuitively, $\bX \subset \bV$ is {\em
preferentially independent\/} of $\bY=\bV-\bX$ if and only if for
all assignments to $\bY$, our preference over $\bX$ values is
identical.
\begin{definition}
\label{d:p1}
Let $\bx_1,\bx_2\in\cD(\bX)$ for some $\bX
\subseteq \bV$, and $\by_1,\by_2\in\cD(\bY)$, where $\bY= \bV -\bX$.
We say that
$\bX$ is {\em preferentially independent} of $\bY$ iff, for all
$\bx_1$, $\bx_2$, $\by_1$, $\by_2$ we have that
\begin{equation}
\label{e:pi} \bx_1\by_1 \succ \bx_2\by_1\ {\mathrm{iff}}\
\bx_1\by_2 \succ \bx_2\by_2
\end{equation}
\end{definition}
For example, in our PC configuration example, the user may assess
{\em screen size\/} to be preferentially independent of {\em processor
type\/} and {\em operating system\/}. This could be the case if
the user always prefers a larger screen to a smaller screen,
independent of the selection of processor and/or OS.

Preferential independence is a strong property, and is therefore not
very common. A more refined notion is that of conditional preferential
independence. Intuitively, $\bX$ is {\em conditionally preferentially
  independent\/} of $\bY$ given $\bZ$ if and only if for every fixed
assignment to $\bZ$, the ranking of $\bX$ values is independent of the
value of $\bY$.
\begin{definition}
\label{d:p2}
Let $\bX,\bY$ and $\bZ$ be a partition of
$\bV$ and let $\bz\in\cD(\bZ)$.  $\bX$ is {\em
  conditionally preferentially independent of $\bY$ given $\bz$} iff, for
all
$\bx_1$, $\bx_2$, $\by_1$, $\by_2$ we have that
\begin{equation}
\label{e:cpi} \bx_1\by_1\bz \succ \bx_2\by_1\bz\ {\mathrm{iff}}\
\bx_1\by_2\bz \succ \bx_2\by_2\bz,
\end{equation}
and $\bX$ is {\em conditionally preferentially independent of $\bY$
given $\bZ$} iff $\bX$ is conditionally preferentially independent of
$\bY$ given every assignment $\bz\in\cD(\bZ)$.
\end{definition}
Returning to our PC example, the user may assess {\em operating
system\/} to be independent of all other features given {\em
processor type\/}. That is, he always prefers LINUX given an AMD
processor and WindowsXP given an Intel processor (e.g., because
he might believe that WindowsXP is optimized for the Intel
processor, whereas LINUX is otherwise better). Note that the
notions of preferential independence and conditional preferential
independence are among a number of standard and well-known notions
of independence in multi-attribute utility theory~\cite{KR}.

\subsection{Relative Importance}
\label{S:importance}
Although statements of preferential independence are natural and
useful, the orderings obtained by relying on them alone are
relatively weak.  To understand this, consider two preferentially
independent boolean attributes $A$ and $B$ with values $a_1,a_2$ and
$b_1,b_2$, respectively.  If $A$ and $B$ are preferentially
independent, then we can specify a preference order over $A$ values,
say $a_1\succ a_2$, independently of the value of $B$. Similarly, our
preference over $B$ values, say $b_1\succ b_2$, is independent of the
value of $A$. From this we can deduce that $a_1b_1$ is the most
preferred outcome and $a_2b_2$ is the least preferred outcome. However,
we do not know the relative order of $a_1b_2$ and $a_2b_1$. This is
typically the case when we consider independent variables: We can rank
each one given a fixed value of the other, but often, we cannot
compare outcomes in which both values are different. One type of
information that can address some (though not necessarily all) such
comparisons is information about relative importance. For instance, if
we state that $A$ is more important than $B$, it means that we
prefer an improvement in $A$ over an improvement in $B$.
In that case, we know that $a_1b_2\succ a_2b_1$, and can
totally order the set of outcomes as
$a_1b_1\succ a_1b_2 \succ a_2b_1 \succ a_2b_2$.

One may ask why it is important for us to order $a_1b_2$ and
$a_2b_1$ -- after all, we know that $a_1b_1$ is the most
preferred outcome. However, in many typical scenarios, we have
auxiliary or user constraints that prevent us from providing the
user with the most preferred (unconstrained) outcome. A simple
and common example is that of budget constraints, other resource
limitations, such are bandwidth and buffer size (as in the
adaptive rich-media systems described by~\citeA{BF05} are also
common. In such cases, it is important to know which attributes
the user cares about more strongly, and to try to maintain good
values for these attributes, compromising on the others.

Returning to our PC configuration example, suppose that the attributes {\em operating
system\/} and {\em processor type\/} are mutually preferentially independent.  We might say that {\em processor type\/} is more important than {\em
operating system\/}, e.g, because we believe that the effect of the
processor's type on system performance is more significant than the
effect of the operating system.
\begin{definition}
\label{d:p3}
Let a pair of variables $X$ and $Y$ be mutually preferentially independent given
$\bW=\bV-\{X,Y\}$. We say that $X$ is {\em more important than}
$Y$, denoted by $X\imp Y$, if for every assignment $\bw\in\cD(\bW)$
and for every $x_i,x_j\in\cD(X)$, $y_a,y_b\in\cD(Y)$,
such that $x_i\succ x_j$ given $\bw$,
we have that:
\begin{equation}
\label{e:ri}
  x_i y_a \bw \succ x_j y_b \bw.
\end{equation}
\end{definition}
Note that Eq.~\ref{e:ri} holds even when $y_b\succ y_a$ given $\bw$.
For instance, when both $X$ and $Y$ are binary variables, and $x_1 \succ
x_2$ and
$y_1 \succ y_2$ hold given $\bw$, then $X \imp Y$ iff we have $x_1 y_2
\bw \succ x_2 y_1 \bw$ for all $\bw\in\cD(\bW)$.
Notice that this is a strict notion of importance -- any reduction in
$Y$ is preferred to any reduction in $X$. Clearly, this idea can be
further refined by providing an actual ordering over elements of $\cD(XY)$, and we discuss this extension
in Section~\ref{ss:relaxing}.
In addition, one can consider relative importance assessments
among more than two variables. However, we feel that the benefit
of capturing such statements is small: We believe that statements
of relative importance referring to more than two attributes are
not very natural for users to articulate, and their inclusion
would significantly reduce the computational advantages of
graphical modeling. Therefore, in this work we focus only on
relative importance statements referring to pairs of attributes.

Relative importance information is a natural enhancement of
independence information.
As such, relative importance retains a desirable property -
it corresponds to statements that a naive user would find simple and
clear to evaluate and articulate. Moreover, it can be generalized
naturally to a notion of {\em conditional relative importance}. For
instance, suppose that the relative importance of {\em processor
type\/} and {\em operating system\/} depends on the primary usage of
the PC. For example, when the PC is used primarily for graphical
applications, then the choice of an operating system is more important
than that of a processor because certain important software packages
for graphic design are not available on LINUX. However, for other
applications, the processor type is more important because
applications for both Windows and LINUX exist. Thus, we say that $X$
is more important than $Y$ given $\bz$ if we always prefer to reduce
the value of $Y$ rather than the value of $X$, whenever $\bz$ holds.
\begin{definition}
\label{d:p4}
Let $X$ and $Y$ be a pair of variables from $\bV$, and let $\bZ\subseteq\bW=\bV-\{X,Y\}$. We say
that $X$ is {\em more important than $Y$ given}
$\bz\in\cD(\bZ)$ iff, for every assignment
$\bw'$ on $\bW'=\bV - (\{X,Y\} \cup \bZ)$ we have:
\begin{equation}
\label{e:cri}
  x_i y_a \bz \bw' \succ x_j y_b \bz \bw'
\end{equation}
whenever $x_i\succ x_j$ given $\bz\bw'$. We denote this relation by $X\imp_{\bz} Y$. Finally, if for some $\bz\in\cD(\bZ)$ we have either
$X\imp_{\bz}Y$, or $Y\imp_{\bz}X$, then we say
that the relative importance of $X$ and $Y$ is conditioned on $\bZ$,
and write $\cR\cI(X,Y | \bZ)$.
\end{definition}

\section{TCP-nets}
\label{S:tcpnet}

The TCP-net (for {\em Tradeoff-enhanced CP-nets\/}) model is an
extension of CP-nets~\cite{BBDHP.journal} that encodes conditional
relative importance statements, as well as the conditional
preference statements supported in CP-nets. The primary usage of
the TCP-net graphical structure is in consistency analysis of the
provided preference statements, and in classification of
complexity and developing efficient algorithms for various
reasoning tasks over these statements. In particular, as we later
show, when this structure is ``acyclic'' (for a suitable
definition of this notion!), the set of preference statements
represented by the TCP-net is guaranteed to be consistent -- that
is, there is a strict total order over the outcomes that
satisfies all the preference statements. In what follows we
formally define the TCP-net model. As it subsumes the CP-net
model, we will immediately define this more general model rather
than proceed in stages.

\subsection{TCP-net Definition}

TCP-nets are annotated graphs with three types of edges.  The nodes of
a \TCPn\ correspond to the problem variables $\bV$. The first type of
(directed) edges comes from the original CP-nets model and captures direct preferential dependencies, that is, such an edge from
$X$ to $Y$ implies that the user has different preferences over $Y$
values given different values of $X$.
The second (directed) edge type captures relative importance
relations. The existence of such an edge from $X$ to $Y$ implies that
$X$ is more important than $Y$.  The third (undirected) edge type
captures conditional importance relations: Such an edge between nodes
$X$ and $Y$ exists if there exists a non-empty variable subset $\bZ \subseteq \bV - \{X,Y\}$ for which
$\cR\cI(X,Y|\bZ)$ holds.  Without loss of generality, in what follows, the set $\bZ$ is assumed to be the minimal set of variables upon which the relative importance between $X$ and $Y$ depends.

As in CP-nets, each node $X$ in a \TCPn\ is annotated with a {\em conditional
  preference table\/} (CPT). This table associates preferences over $\cD(X)$
for every possible value assignment to the parents of $X$ (denoted $Pa(X)$).
In addition, in \TCPns, each undirected edge is annotated with a {\em
  conditional importance table\/} (CIT).  The CIT associated with such an edge
$(X,Y)$ describes the relative importance of
$X$ and $Y$ given the
value of the corresponding importance-conditioning variables $\bZ$.
\begin{definition}
\label{d:tcpnet}
A \TCPn\ $\cN$ is a tuple $\langle \bV,{\sf cp},{\sf i},{\sf
  ci},{\sf cpt},{\sf cit} \rangle$ where:
\begin{enumerate}[(1)]
  \item $\bV$ is a set of nodes, corresponding to the problem
  variables $\{X_1,\ldots,X_n\}$.

\item ${\mathsf{cp}}$ is a set of directed {\em ${\mathsf{cp}}$-arcs}
  $\{\alpha_1,\ldots,\alpha_k\}$ (where ${\mathsf{cp}}$ stands for
  {\em conditional preference}).  A ${\mathsf{cp}}$-arc
  $\cparc{X_i}{X_j}$ is in $\cN$ iff the preferences over the
  values of $X_j$ depend on the actual value of $X_i$.
  For each $X \in \bV$, let $Pa(X) = \{X' | \cparc{X'}{X} \in {\mathsf{cp}}\}$.

\item ${\mathsf{i}}$ is a set of directed {\em ${\mathsf{i}}$-arcs}
  $\{\beta_1,\ldots,\beta_l\}$ (where ${\mathsf{i}}$ stands for {\em
    importance}).  An ${\mathsf{i}}$-arc $\iarc{X_i}{X_j}$ is in
  $\cN$ iff $X_i \imp X_j$.

\item ${\mathsf{ci}}$ is a set of undirected {\em ${\mathsf{ci}}$-arcs}
  $\{\gamma_1,\ldots,\gamma_m\}$ (where ${\mathsf{ci}}$ stands for {\em
    conditional importance}).  A ${\mathsf{ci}}$-arc $\ciarc{X_i}{X_j}$
    is in $\cN$ iff we have $\cR\cI(X_i,X_j|\bZ)$ for some $\bZ \subseteq \bV -
  \{X_i,X_j\}$.\footnote{Observe that every ${\mathsf{i}}$-arc
    $\iarc{X_i}{X_j}$ can be seen as representing $\cR\cI(X_i,X_j|\emptyset)$.
    However, a clear distinction between ${\mathsf{i}}$-arcs and
    ${\mathsf{ci}}$-arc simplifies specification of many forthcoming notions and claims (e.g., Lemma~\ref{l:aux2} in Section~\ref{S:acyclic}, as well as the related notion of root variables.)}
     We call $\bZ$ the {\bf {\em selector set\/}} of $\ciarc{X_i}{X_j}$ and denote it
  by $\cS(X_i,X_j)$.

\item ${\mathsf{cpt}}$ associates a CPT with every node $X\in\bV$, where $CPT(X)$ is a
mapping from $\cD(Pa(X))$ (i.e., assignments to $X$'s parent
nodes) to strict partial orders over $\cD(X)$.

\item ${\mathsf{cit}}$ associates with every ${\mathsf{ci}}$-arc
  $\gamma = (X_i,X_j)$ a (possibly partial)
  mapping $CIT(\gamma)$ from
   $\cD\left(\cS(X_i,X_j)\right)$
   to orders over the set $\{X_i,X_j\}$.%
   \footnote{That is, the relative importance relation between
   $X_i$ and $X_j$ may be specified only for certain values of the
   selector set.}
\end{enumerate}
\end{definition}

A TCP-net in which the sets ${\mathsf{i}}$ and ${\mathsf{ci}}$ (and therefore also
${\mathsf{cit}}$) are empty, is also a CP-net. Thus, it is
the elements ${\mathsf{i}}$, ${\mathsf{ci}}$, and
${\mathsf{cit}}$ that describe absolute and conditional importance of attributes provided by TCP-nets,
beyond the conditional preference information captured by CP-nets.

\subsection{TCP-net Semantics}
\label{S:semantics} 
The semantics of a \TCPn\ is  defined in terms of the set of strict partial orders
consistent with the set of constraints imposed by the preference
and importance information captured by this TCP-net. The intuitive
idea is rather straightforward:
(1) A strict partial order
$\succ$ satisfies the conditional preferences for variable $X$ if
any two complete assignments that differ only on the value of $X$ are
ordered by $\succ$ consistently with the ordering on $X$ values
in the CPT of $X$. Recall that this ordering can depend on the
parent of $X$ in the graph. (2) A strict partial order $\succ$
satisfies the assertion that $X$ is more important than $Y$ if
given any two complete assignments that differ on the value of $X$
and $Y$ only, $\succ$ prefers that assignment which provides $X$
with a better value. (3) A strict partial order $\succ$ satisfies
the assertion that $X$ is more important than $Y$ given some
assignment $z$ to variable set $Z$ if given any two complete
assignments that differ on the value $X$ and $Y$ only, and in
(both of) which $Z$ is assigned $z$, $\succ$ prefers that
assignment which provides $X$ with a better value.

This is defined more formally below. We use $\succ^{X}_{\bu}$ to
denote the preference relation over the values of $X$ given an
assignment $\bu$ to $\bU\supseteq Pa(X)$.

\begin{definition}
  \label{def:tcp-satisf}
  Consider a TCP-net $\cN = \langle \bV,{\sf cp},{\sf i},{\sf
  ci},{\sf cpt},{\sf cit} \rangle$.
\begin{enumerate}
\item Let $\bW = \bV - (\{X\} \cup Pa(X))$ and let $\bp\in\cD(Pa(X))$.
  A preference (=strict partial) order $\succ$ over $\cD(\bV)$ satisfies $\succ^X_{\bp}$ iff $x_i \bp
  \bw \succ x_j \bp \bw$, for every $\bw\in\cD(\bW)$, whenever
  $x_i \succ^X_{\bp} x_j$ holds.

\item A preference order $\succ$ over $\cD(\bV)$ satisfies $CPT(X) \in {\sf cpt}$  iff it
  satisfies $\succ^X_{\bp}$ for every assignment $\bp$ of $Pa(X)$.

\item A preference order $\succ$ over $\cD(\bV)$ satisfies $X \imp Y$ iff for every
  $\bw\in\cD(\bW)$ such that $\bW=\bV-\{X,Y\}$, 
  $x_i y_a \bw \succ x_j y_b  \bw$ whenever
  $x_i \succ^X_{\bw} x_j$.

\item A preference order $\succ$ over $\cD(\bV)$ satisfies $X \imp_{\bz}Y$ iff for
  every $\bw\in\cD(\bW)$ such that $\bW=\bV-(\{X,Y\}\cup \bZ)$, 
  $x_i y_a \bz\bw \succ x_j y_b \bz\bw$ whenever $x_i \succ^X_{\bz\bw} x_j$.

\item A preference order $\succ$ over $\cD(\bV)$ satisfies $CIT(\gamma)$ of the
${\mathsf{ci}}$-arc
  $\gamma=\ciarc{X}{Y} \in {\sf cit}$
  if it satisfies $X\imp_{\bz}Y$ whenever an
  entry in the table conditioned on $\bz$ ranks $X$ as more important.
\end{enumerate}

A {\em preference order $\succ$ over $\cD(\bV)$ satisfies a
TCP-net}
  $\cN = \langle \bV,{\sf cp},{\sf i},{\sf
  ci},{\sf cpt},{\sf cit} \rangle$ iff:
\begin{enumerate}[(1)]
\item for every $X \in \bV$, $\succ$ satisfies $CPT(X)$,
\item for every ${\mathsf{i}}$-arc $\beta = \iarc{X_i}{X_j} \in {\sf i}$, $\succ$ satisfies $X \imp Y$, and
\item for every ${\mathsf{ci}}$-arc $\gamma = \ciarc{X_i}{X_j} \in {\sf ci}$, $\succ$ satisfies $CIT(\gamma)$.
\end{enumerate}
\end{definition}

\begin{definition}
  \label{def:tcp-satisf2}
  A \TCPn\ is {\em satisfiable} iff there is some strict partial order $\succ$ over $\cD(\bV)$ that
  satisfies it; $o\succ o'$ is {\em implied\/} by a \TCPn\ $\cN$ iff it holds
  in all preference orders over $\cD(\bV)$ that satisfy $\cN$.
\end{definition}

\begin{lemma}
  \label{l:tcp-transitive}
  Preferential entailment with respect to a satisfiable TCP-net is transitive.
  That is, if $\cN \models o\succ o'$ and $\cN \models o'\succ o''$, then $\cN
  \models o\succ o''$.
\end{lemma}

\begin{proof}
  If $\cN \models o\succ o'$ and $\cN \models o'\succ o''$, then
  $o\succ o'$ and $o'\succ o''$ in all preference orders satisfying $N$.
  As each of these ordering is transitive, we must have $o\succ o''$
  in all satisfying orderings.
\end{proof}

Note that, strictly speaking, we should use the term ``satisfiable'' rather
than ``consistent'' with respect to a set of preference statements, given that
we provide a model theory, and not a proof theory. However, since
the corresponding proof theory follows in a completely straightforward manner from our
semantics combined with transitivity, this raises no problem.

\subsection{TCP-net Examples}

Having provided the formal specification of the TCP-nets model,
let us now illustrate TCP-nets with a few examples.
For simplicity of presentation, in the following examples all variables are binary, although the semantics of \TCPns\ is
given by Definitions~\ref{def:tcp-satisf} and~\ref{def:tcp-satisf2} with respect to arbitrary finite domains.

\begin{example}[{\sf Evening Dress}]
  \label{e:example1}
  Figure~\ref{f:example1}(a) presents a CP-net that consists of three
  variables $J$, $P$, and $S$, standing for the jacket, pants, and
  shirt, respectively.  I prefer black to white as a color
  for both the jacket and the pants, while my preference for the shirt
  color (red/white) is conditioned on the {\em color combination} of
  jacket and pants: If they are of the same color, a white shirt will
  make my dress too colorless, therefore, red shirt is preferable.
  Otherwise, if the jacket and the pants are of different colors, a
  red shirt will probably make my evening dress too flashy, therefore,
  a white shirt is preferable. The solid lines in
  Figure~\ref{f:example1}(c) show the preference relation induced
  directly by the information captured by this CP-net; The top and the
  bottom elements are the worst and the best outcomes, respectively,
  and the arrows are directed from less preferred to more preferred
  outcomes.

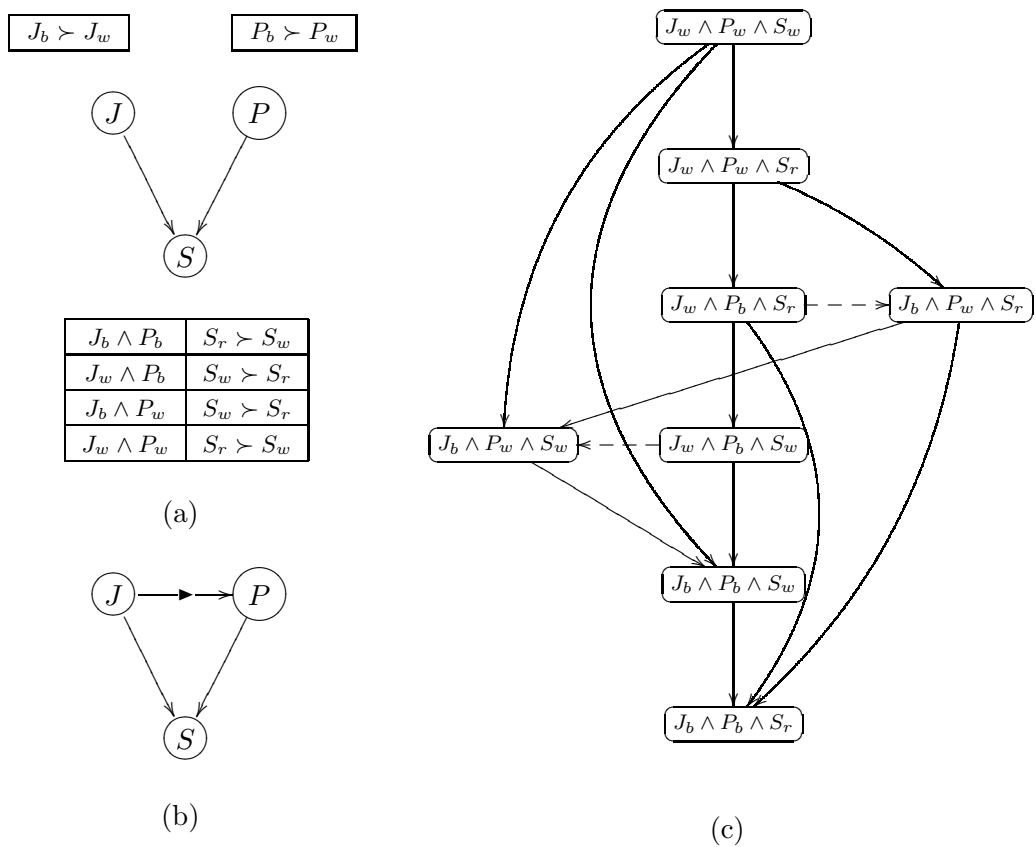
\begin{figure}[ht]
    \mbox{
    \begin{tabular}{cc}
      \begin{minipage}{2in}
      \begin{tabular}{ccc}
        \fbox{
          {\footnotesize
          $J_b \succ J_w$
          }} & &
        \fbox{
          {\footnotesize
          $P_b \succ P_w$
          }}\\
        \ \\
      \multicolumn{3}{c}{
      \def\objectsizestyle{\scriptstyle}
      \xymatrix @R=35pt@C=8pt{
        *++[o][F]{J} \ar[dr] & & *++[o][F]{P}\ar[dl]\\
        & *++[o][F]{S}
        }}\\
        \ \\
        \multicolumn{3}{c}{
          {\footnotesize
            \setlength{\extrarowheight}{2pt}
        \begin{tabular}{|c|c|}
          \hline
          $J_b \wedge P_b$ & $S_r \succ S_w$\\
          \hline
          $J_w \wedge P_b$ & $S_w \succ S_r$\\
          \hline
          $J_b \wedge P_w$ & $S_w \succ S_r$\\
          \hline
          $J_w \wedge P_w$ & $S_r \succ S_w$\\
          \hline
        \end{tabular}
        }}\\
        \ \\
        & (a)\\
        \ \\
        \multicolumn{3}{c}{
        \def\objectsizestyle{\scriptstyle}
        \xymatrix @R=35pt@C=8pt{
          *++[o][F]{J} \ar[dr] \ar[rr]|{\RHD} & & *++[o][F]{P}\ar[dl]\\
          & *++[o][F]{S}
          }}\\
        \ \\
        & (b)\\
      \end{tabular}
      \end{minipage} &
      \begin{minipage}{1.5in}
        \begin{tabular}{c}
          {\scriptsize
          \def\objectsizestyle{\scriptstyle}
          \xymatrix @R=40pt@C=30pt{
            & *+[F-:<3pt>]{J_w \wedge P_w \wedge S_w} \ar[d]
\ar@/_2pc/[dddl]
              \ar@/_4.5pc/[dddd]\\
            & *+[F-:<3pt>]{J_w \wedge P_w \wedge S_r} \ar[d]
\ar@/^0.5pc/[dr]\\
            & *+[F-:<3pt>]{J_w \wedge P_b \wedge S_r} \ar[d]
\ar@/^2.7pc/[ddd] \ar@{-->}[r]&
              *+[F-:<3pt>]{J_b \wedge P_w \wedge S_r} \ar@/^1.5pc/[dddl]
\ar[dll] \\
            *+[F-:<3pt>]{J_b \wedge P_w \wedge S_w} \ar[dr] &
              *+[F-:<3pt>]{J_w \wedge P_b \wedge S_w} \ar[d] \ar@{-->}[l] \\
            & *+[F-:<3pt>]{J_b \wedge P_b \wedge S_w} \ar[d]\\
            & *+[F-:<3pt>]{J_b \wedge P_b \wedge S_r}
            }}\\
           \ \\
           \ \\
           (c)
        \end{tabular}
      \end{minipage}
    \end{tabular}
    }
  \caption{``Evening Dress'' CP-net \& TCP-net.} \label{f:example1}
\end{figure}

  Figure~\ref{f:example1}(b) depicts a \TCPn\ that extends
  this CP-net by adding an ${\mathsf{i}}$-arc from $J$ to $P$, i.e.,
  having black jacket is (unconditionally) more important than having black
  pants. This induces additional relations among outcomes, captured by
  the dashed lines in Figure~\ref{f:example1}(c). $\Diamond$
\end{example}

The reader may rightfully ask whether the statement of importance
in Example~\ref{e:example1} is not redundant: According to my
preference, it seems that I will always wear a black suit with a
red shirt. However, while my preferences are clear, various
constraints may make some outcomes, including the most preferred
one, infeasible. For instance, I may not have a clean black
jacket, in which case the most preferred feasible alternative is a
white jacket, black pants, and a white shirt. Alternatively,
suppose that the only clean clothes I have are velvet black
jacket and white pants, and silk white jacket and black pants. My
wife forbids me to mix velvet and silk, and so I will have to
compromise, and to wear either the black (velvet) jacket with the
white (velvet) pants, or the white (silk) jacket with the black
(silk) pants. In this case, the fact that I prefer wearing the
preferred jacket to wearing the preferred pants determines higher
desirability for the velvet combination. Now, if my wife has to
prepare my evening dress while I am late at work writing a paper,
having this information will help her to choose among the {\em
available\/}  options an outfit that I would like most.

Indeed, as noted earlier, many applications involve 
limited resources, such as money, time, bandwidth, memory, etc. In
many instances, the optimal assignment violates these resource
constraints, and we must compromise and accept a less desirable,
but feasible assignment. TCP-nets capture information that allows
us make more informed compromises.

\begin{example}[{\sf Flight to the USA}]
  \label{e:example2}
  Figure~\ref{f:example2}(a) illustrates a more complicated CP-net,
  describing my preference over the flight options to a conference in
  the USA, from Israel. This network consists of five variables, standing
  for various parameters of the flight:
  \begin{description}
  \item[\underline{D}ay of the Flight] The variable $D$ distinguishes
    between flights leaving a day ($D_{1d}$) and two days ($D_{2d}$)
    before the conference, respectively.  Since I am married, and I am
    really busy with my work, I prefer to leave on the day before the
    conference.

  \item[\underline{A}irline] The variable $A$ represents the airline.
    I prefer to fly with British Airways ($A_{ba}$) than with KLM
($A_{klm}$).

  \item[Departure \underline{T}ime] The variable $T$ distinguishes
    between morning/noon ($T_m$) and evening/night ($T_n$) flights.
    Among flights leaving two days before the conference I prefer
    an evening/night flight, because it will allow me to work
    longer on the day of the flight.
    However, among flights leaving a day before the
    conference I prefer a morning/noon flight, because I would like to
    have a few hours before the conference opening in order to     rest at the hotel.

  \item[\underline{S}top-over] The variable $S$ distinguishes between direct
$(S_{0s})$
    and indirect $(S_{1s})$ flights, respectively.
    On day flights I am awake most of the time and, being a smoker, prefer
    a stop-over in Europe (so I can have a smoking break).
    However, on night flights I sleep, leading to a preference for direct
    flights, since they are shorter.

  \item[Ticket \underline{C}lass] The variable $C$ stands for ticket class.
    On a night flight, I prefer to sit in economy class
    ($C_e$) (I don't care where I sleep, and these seats are
    significantly cheaper), while on a day flight I prefer to pay for
    a seat in business class ($C_b$) (Being awake, I can better
    appreciate the good seat, food, and wine).
  \end{description}

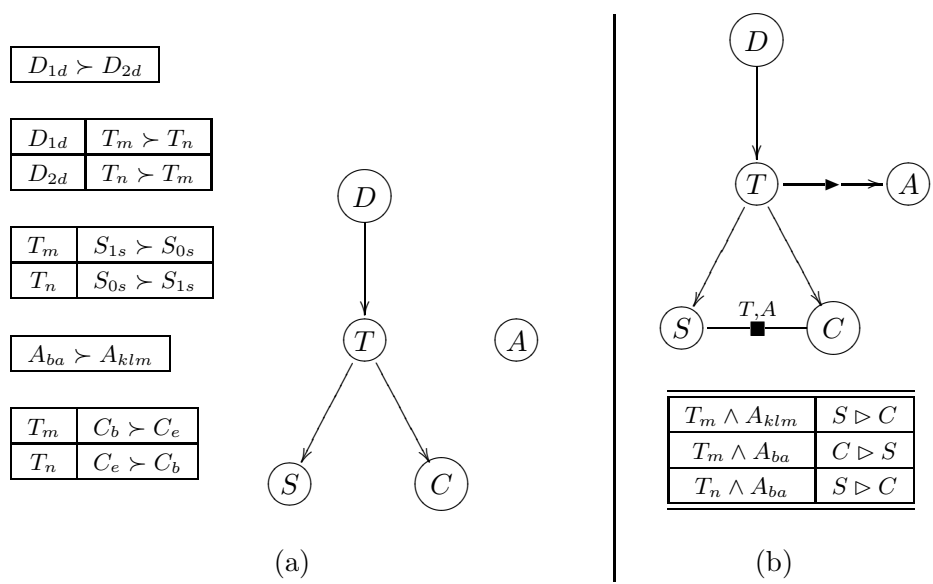
\begin{figure}[ht]
    \begin{center}
      \begin{tabular}{c|c}
        \begin{minipage}{3.2in}
        \begin{tabular}{ll}
          {\footnotesize
            \setlength{\extrarowheight}{2pt}
          \begin{tabular}{|c|}
            \hline
            $D_{1d} \succ D_{2d}$\\
            \hline
          \end{tabular}
          } &
          \multirow{5}{2in}
          {
            \def\objectsizestyle{\scriptstyle}
            \xymatrix @R=35pt@C=8pt{
              \\
              & *++[o][F]{D} \ar[d]\\
              & *++[o][F]{T} \ar[dl] \ar[dr] & & *++[o][F]{A}\\
              *++[o][F]{S} & & *++[o][F]{C}
              }
            }\\
          \ \\
          {\footnotesize
            \setlength{\extrarowheight}{2pt}
          \begin{tabular}{|c|c|}
            \hline
            $D_{1d}$ & $T_m \succ T_n$\\
            \hline
            $D_{2d}$ & $T_n \succ T_m$\\
            \hline
          \end{tabular}
          }\\
          \ \\
          {\footnotesize \setlength{\extrarowheight}{2pt}
          \begin{tabular}{|c|c|}
            \hline
            $T_m$ & $S_{1s} \succ S_{0s}$\\
            \hline
            $T_n$ & $S_{0s} \succ S_{1s}$\\
            \hline
          \end{tabular}
          }\\
          \ \\{\footnotesize
            \setlength{\extrarowheight}{2pt}
          \begin{tabular}{|c|}
            \hline
            $A_{ba} \succ A_{klm}$\\
            \hline
          \end{tabular}
          } \\
          \ \\
          {\footnotesize \setlength{\extrarowheight}{2pt}
          \begin{tabular}{|c|c|}
            \hline
            $T_m$ & $C_b \succ C_e$\\
            \hline
            $T_n$ & $C_e \succ C_b$\\
            \hline
          \end{tabular}
          }
        \end{tabular}
        \end{minipage} & 
        \begin{minipage}{1.5in}
          \begin{tabular}{c}
            \def\objectsizestyle{\scriptstyle}
            \xymatrix @R=35pt@C=8pt{
              & *++[o][F]{D} \ar[d]\\
              & *++[o][F]{T} \ar[dl] \ar[dr] \ar[rr]|{\RHD} & &
*++[o][F]{A}\\
              *++[o][F]{S} \ar@{-}[rr]|{\blacksquare}^{T,A} & & *++[o][F]{C}
              }\\
            \ \\
            {\footnotesize \setlength{\extrarowheight}{2pt}
            \begin{tabular}{|c|c|}
             \hline
              \hline
              $T_m \wedge A_{klm}$ & $S \imp C$\\
              \hline
              $T_m \wedge A_{ba}$ & $C \imp S$\\
              \hline
              $T_n \wedge A_{ba}$ & $S \imp C$\\
              \hline
        \hline
            \end{tabular}
            }
          \end{tabular}
        \end{minipage}\\
        \ \\
        (a) & (b)
      \end{tabular}
\end{center}
\caption{``Flight to the USA'' CP-net \& TCP-net from
Example~\ref{e:example2}.} \label{f:example2}
\end{figure}

  The CP-net in Figure~\ref{f:example2}(a) captures all these preference
  statements, and the underlying preferential dependencies, while
  Figure~\ref{f:example2}(b) presents a \TCPn\ that extends this
  CP-net to capture relative importance relations between some
  parameters of the flight. First, there is an ${\mathsf{i}}$-arc from
  $T$ to $A$, because getting more suitable flying time is more
  important to me than getting the preferred airline.
  Second, there is a ${\mathsf{ci}}$-arc between $S$ and $C$, where the
  relative importance of $S$ and $C$ depends on the values of $T$ and $A$:\footnote{For clarity, the ${\mathsf{ci}}$-arc in Figure~\ref{f:example2}(b) is {\em schematically} labeled with its importance-conditioning variables $T$ and $A$.}
  \begin{enumerate}
  \item On a KLM day flight, an intermediate
    stop in Amsterdam is more important to me than flying
    business class (I feel that KLM's business class does not have a
    good cost/performance ratio, while visiting a casino in
    Amsterdam's airport sounds to me like a good idea.)

  \item For a British Airways night flight, the fact that
    the flight is direct is more important to me than getting a cheaper
economy
    seat (I am ready to pay for business class,
    in order not to spend even one minute at Heathrow airport at night).

  \item On a British Airways day flight,
    business class is more important to me than having a
    short intermediate break (it is hard to find a nice smoking
    area at Heathrow).
  \end{enumerate}

  The CIT of this ${\mathsf{ci}}$-arc is also shown in
  Figure~\ref{f:example2}(b). $\Diamond$
\end{example}

\begin{figure}[ht]
\begin{center}
\begin{minipage}{\textwidth}
\begin{tabular}{rcl}
                {\footnotesize
                \setlength{\extrarowheight}{2pt}
          \begin{tabular}{|c|}
            \hline
            $D_{1d} \succ D_{2d}$\\
            \hline
          \end{tabular}
          }
          \vspace{1.5cm}
      &
         \multirow{3}{1.3in}{
         \mbox{
            \xymatrix @R=35pt@C=15pt{
              *++[o][F]{D} \ar[d]  & & *++[o][F]{A} \ar[ddl]\\
              *++[o][F]{T} \ar[dr] \\
               & *++[F-]{SC}
              }
          }
         }
      &
         {\footnotesize
            \setlength{\extrarowheight}{2pt}
          \begin{tabular}{|c|}
            \hline
            $A_{ba} \succ A_{klm}$\\
            \hline
          \end{tabular}
          }
      \\
          {\footnotesize
            \setlength{\extrarowheight}{2pt}
          \begin{tabular}{|c|c|}
            \hline
            $D_{1d}$ & $T_m \succ T_n$\\
            \hline
            $D_{2d}$ & $T_n \succ T_m$\\
            \hline
          \end{tabular}
          }
          \vspace{1.5cm}
       & &
       \multirow{2}{2in}{
       {\footnotesize \setlength{\extrarowheight}{2pt}
          \begin{tabular}{|c|c|}
            \hline
            $T_m \wedge A_{ba}$ & $S_{1s}C_{b} \succ S_{0s}C_{b} \succ S_{1s}C_{e} \succ S_{0s}C_{e}$\\
            \hline
            $T_m \wedge A_{klm}$ & $S_{1s}C_{b}  \succ S_{1s}C_{e} \succ S_{0s}C_{b} \succ S_{0s}C_{e}$\\
            \hline
            $T_{n} \wedge A_{ba}$ & $S_{0s}C_{e} \succ S_{0s}C_{b} \succ S_{1s}C_{e} \succ S_{1s}C_{b}$\\
            \hline
            $T_{n} \wedge A_{ba}$ &
            $S_{0s}C_{e} \succ S_{0s}C_{b} \succ S_{1s}C_{b}$\\
             & $S_{0s}C_{e} \succ S_{1s}C_{e} \succ S_{1s}C_{b}$\\
            \hline
          \end{tabular}
          }
       }\\
       & & \\
\end{tabular}
\end{minipage}
\end{center}
\caption{The network obtained by clustering variables $S$ and $C$ in Example~\ref{e:example2}.}
\label{fig:cluster}
\end{figure}
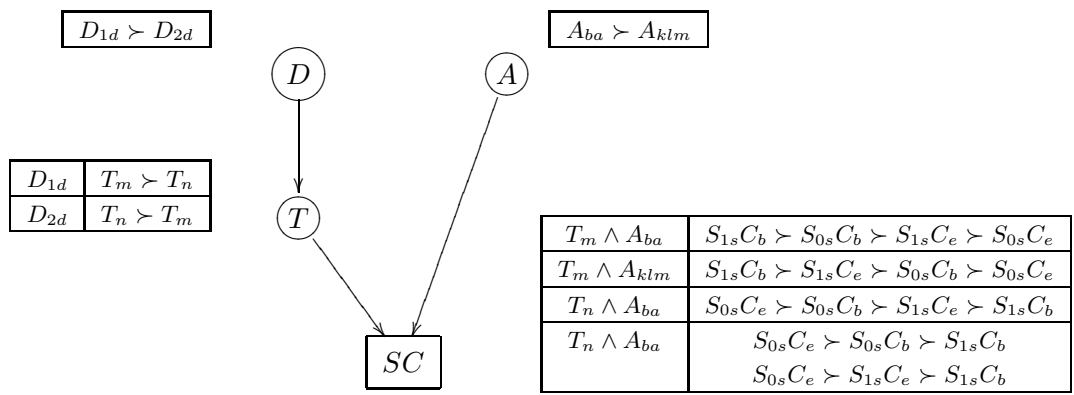

\subsection{Relative Importance with Non-binary Variables}
\label{ss:relaxing}
Having read so far, the reader may rightfully ask
whether the notion of relative (conditional) importance {\em ceteris
  paribus}, as specified in Section~\ref{S:importance} (Eq.~\ref{e:ri} and~\ref{e:cri}), is not
too strong when the variables are not binary. For example, consider a
more refined notion of departure time (variable $T$) in Example~\ref{e:example2}, and suppose there
are more than two companies flying from Israel to the USA (variable
$A$). In this case, one may prefer a better flight time, even
if this requires a compromise in the airline, {\em as long as} this
compromise is not too significant. For instance, to get a better flight time, one may be willing
to compromise and accept any airline but only among those she ranks in the top $i$ places in
this context.

More generally, our notion of importance, as well as some more refined notions of it, 
are really means of specifying an ordering
over assignments to variable pairs. In a sense, one could reduce
TCP-nets into CP-nets 
by combining variables between which we have an importance relation.
Thus, for instance, in the
``Flight to the USA'' example, we could combine the variables $S$
and $C$ (see Figure~\ref{fig:cluster}). The resulting variable, $SC$ will have as its domain the
Cartesian product of the domains of $S$ and $C$. The preferences
for the values of $SC$ are now conditioned on $T$, the current
parent of $S$ and $C$, as well as on $A$, which belongs to the
selector set of their CIT. In general, the selector set (and parents
of) a pair of variables can be viewed as conditioning the
preferences over the value combinations for this pair.
Hence, such clustering can help us already in the case of binary
variables as certain orderings over the assignments to two binary variables
cannot be specified with a TCP-net. However, this
is clearly more of an issue in the case of non-binary variables, where
the number of combinations of pairs of values is much larger.

The bottom line is that more complex importance relations between
pairs of variables can be captured. The main questions is how.
The strict importance relation we use captures certain such
relations in a very compact manner. As such, its specification
(e.g., in terms of natural language statements) is very easy.
This does not rule out the possibility of expressing more refined
relations. Various linguistic constructs could be used to express
such relations. However, technically, they can all be captured by
clustering the relevant variables, and the resulting
representation would be a TCP-net, or possibly simply a CP-net.
Of course, it is quite possible that some relations have an
alternative compact representation that could help make reasoning
with them more efficient than simply collapsing them, and this can
be a useful question for future research to examine.

\section{Conditionally Acyclic TCP-nets}
\label{S:acyclic}

Returning to the notion of TCP-net satisfiability, observe that
Definition~\ref{def:tcp-satisf2} provides no practical tools for
verifying satisfiability of a given TCP-net.
Tackling this issue, in this section we introduce a large
class of TCP-nets whose members are guaranteed to be satisfiable.
We refer to this class of TCP-nets as {\em
  conditionally acyclic}. 

Let us begin with the notion of the dependency graph induced by a
\TCPn.

\begin{figure}[t]
  \begin{center}
    \hbox{
    \begin{tabular}{cc}
      \begin{minipage}{1in}
        \begin{tabular}{c}
          {\footnotesize
          \def\objectsizestyle{\scriptstyle}
          \xymatrix @R=35pt@C=8pt{
            & *++[o][F]{D} \ar[d]\\
            & *++[o][F]{T} \ar[dl] \ar[dr] \ar[rr]|{\RHD} & & *++[o][F]{A}\\
            *++[o][F]{S} \ar@{-}[rr]|{\blacksquare}^{T,A} & & *++[o][F]{C}
            }
          }\\
          \ \\
          {\scriptsize \setlength{\extrarowheight}{2pt}
            \begin{tabular}{|c|c|}
              \hline
              \hline
              $T_m \wedge A_{klm}$ & $S \imp C$\\
              \hline
              $T_m \wedge A_{ba}$ & $C \imp S$\\
              \hline
              $T_n \wedge A_{ba}$ & $S \imp C$\\
              \hline
              \hline
            \end{tabular}
            }
        \end{tabular}
      \end{minipage} &
      \begin{minipage}{1.5in}
        \begin{tabular}{c}
          {\footnotesize
          \def\objectsizestyle{\scriptstyle}
          \xymatrix @R=35pt@C=8pt{
            & *++[o][F]{D} \ar[d]\\
            & *++[o][F]{T} \ar[dl] \ar[dr] \ar[rr] & & *++[o][F]{A} \ar[dl]
\ar[dlll]\\
            *++[o][F]{S} \ar@{-}[rr] & & *++[o][F]{C}
            }
          }
        \end{tabular}\\
        \ \\
        \ \\
        \ \\
      \end{minipage}\\
      \ \\
      (a) & (b)\\
      \ \\
      \multicolumn{2}{c}{
      \begin{tabular}{cccc}
        \begin{minipage}{1.2in}
          \begin{tabular}{c}
            {\footnotesize
              \def\objectsizestyle{\scriptstyle}
              \xymatrix @R=15pt@C=5pt{
                & *+[o][F]{D} \ar[d]\\
                & *+[o][F]{T} \ar[dl] \ar[dr] \ar[rr] & & *+[o][F]{A} \ar[llld] \ar[ld]\\
                *+[o][F]{S} \ar[rr] & & *+[o][F]{C}
              }
            }
          \end{tabular}
        \end{minipage} &
        \begin{minipage}{1.2in}
          \begin{tabular}{c}
            {\footnotesize
              \def\objectsizestyle{\scriptstyle}
              \xymatrix @R=15pt@C=5pt{
                & *+[o][F]{D} \ar[d]\\
                & *+[o][F]{T} \ar[dl] \ar[dr] \ar[rr] & & *+[o][F]{A} \ar[llld] \ar[ld]\\
                *+[o][F]{S} & & *+[o][F]{C}
              }
            }
          \end{tabular}
        \end{minipage} &
        \begin{minipage}{1.2in}
          \begin{tabular}{c}
            {\footnotesize
              \def\objectsizestyle{\scriptstyle}
              \xymatrix @R=15pt@C=5pt{
                & *+[o][F]{D} \ar[d]\\
                & *+[o][F]{T} \ar[dl] \ar[dr] \ar[rr] & & *+[o][F]{A} \ar[llld] \ar[ld]\\
                *+[o][F]{S} & & *+[o][F]{C} \ar[ll]
              }
            }
          \end{tabular}
        \end{minipage} &
        \begin{minipage}{1.2in}
          \begin{tabular}{c}
            {\scriptsize
              \def\objectsizestyle{\scriptstyle}
              \xymatrix @R=15pt@C=3pt{
                & *+[o][F]{D} \ar[d]\\
                & *+[o][F]{T} \ar[dl] \ar[dr] \ar[rr] & & *+[o][F]{A} \ar[llld] \ar[ld]\\
                *+[o][F]{S} & & *+[o][F]{C} \ar[ll]
              }
            }
          \end{tabular}
        \end{minipage}\\
        {\scriptsize $(T_m\wedge A_{klm})$-directed} &
        {\scriptsize $(T_n\wedge A_{klm})$-directed} &
        {\scriptsize $(T_m\wedge A_{ba})$-directed} &
        {\scriptsize $(T_n\wedge A_{ba})$-directed}\\
        \ \\
      \multicolumn{4}{c}{
        (c)
        }
      \end{tabular}
      }
    \end{tabular}
    }
  \end{center}
  \vspace{-0.5cm}
  \caption{(a) ``Flight to USA'' TCP-net. (b) Its dependency graph. (c), Four
    $\bw$-directed graphs.}
  \label{fig:dependency-flight}
\end{figure}
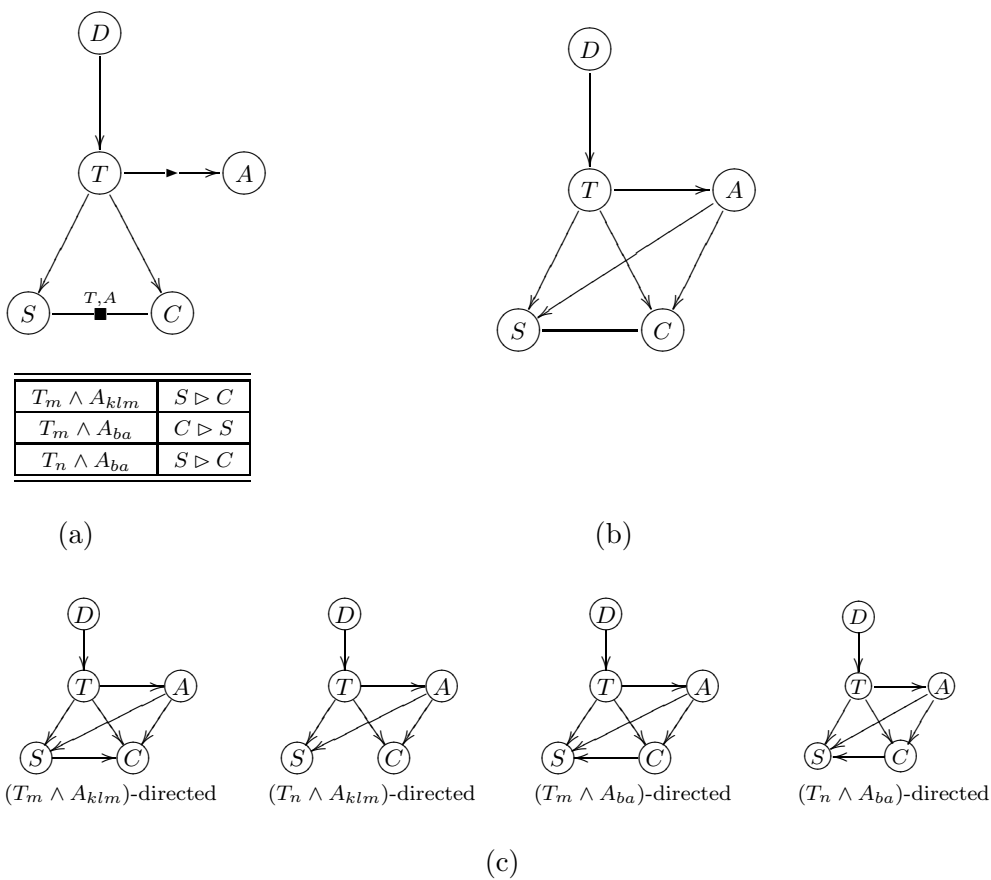

\begin{definition}
  \label{def:tcp-dependency}
  %
  \commentout{
  The {\em dependency graph\/} $\cNN$ of TCP-net $\cN$ contains all the nodes
  and edges of $\cN$, as well as a pair of directed edges $(X_k,X_i)$ and
  $(X_k,X_j)$ for every ${\mathsf{ci}}$-arc $\ciarc{X_i}{X_j}$ in
  $\cN$ and every $X_k\in\cS(X_i,X_j)$.
  }
  The {\em dependency graph\/} $\cNN$ of TCP-net $\cN$ contains all the nodes
  and edges of $\cN$. Additionally, for every ${\mathsf{ci}}$-arc $\ciarc{X_i}{X_j}$ in
  $\cN$ and every $X_k\in\cS(X_i,X_j)$, $\cNN$ contains  a pair of directed edges $(X_k,X_i)$ and
  $(X_k,X_j)$, if these edges are not already in $\cN$.
\end{definition}

Figure~\ref{fig:dependency-flight}(b) depicts the dependency
graph of the TCP-net from the ``Flight to USA'' example, repeated
for convenience in Figure~\ref{fig:dependency-flight}(a). For the
next definition, recall that the {\em selector set\/} of a
${\mathsf{ci}}$-arc is the set of nodes whose value determines the
``direction'' of this arc. Recall also, that once we assign a
value to the selector set, we are, in essence, orienting all the
conditional importance edges. More generally, once all selector
sets are assigned, we transform both $\cN$ and $\cNN$. This
motivates the following definition.

\begin{definition}
  \label{def:tcp-w-directed}
  \commentout{
Let $\cS(\cN)$ be the union of all selector sets of $\cN$.  Given an
assignment $\bw$ to all nodes in $\cS(\cN)$, the {\em \bw-directed graph}
of  $\cNN$ consists of all the nodes and directed edges of $\cNN$,
extended with an edge from $X_i$ to $X_j$ if $(X_i,X_j)$ is a ${\mathsf{ci}}$-arc of $\cN$
and the CIT for $(X_i,X_j)$ specifies that $X_i \imp X_j$ given $\bw$.
}
Let $\cS(\cN)$ be the union of all selector sets of $\cN$.  Given an
assignment $\bw$ to all nodes in $\cS(\cN)$, the {\em \bw-directed graph}
of  $\cNN$ consists of all the nodes and directed edges of $\cNN$. In addition it has
a directed edge from $X_i$ to $X_j$ if such an edge is not already in
$\cNN$, and $(X_i,X_j)$ is a ${\mathsf{ci}}$-arc of $\cN$
and the CIT for $(X_i,X_j)$ specifies that $X_i \imp X_j$ given $\bw$.
\end{definition}

Figure~\ref{fig:dependency-flight}(c) presents all the four $\bw$-directed
graphs of the TCP-net from the ``Flight to USA'' example. Note that,
for the KLM night flights, the relative importance
of $S$ and $C$ is not specified, thus there is no edge between $S$ and $C$
in the $(T_n\wedge A_{klm})$-directed graph of $\cNN$.

Using Definitions~\ref{def:tcp-dependency}
and~\ref{def:tcp-w-directed}, we specify the class of
conditionally acyclic \TCPns, and show that it is
satisfiable\footnote{The authors would like to thank Nic Wilson
for pointing out an error in the original definition of
conditionally acyclic TCP-nets in~\cite{Brafman:Domshlak:uai02}.}.

\begin{definition}
\label{d:adag}
A \TCPn\ $\cN$ is {\em conditionally acyclic} if, for every assignment $\bw$ to $\cS(\cN)$,
the induced $\bw$-directed graphs of $\cNN$ are acyclic.
\end{definition}

We now show that every conditionally acyclic \TCPn\ is satisfiable,
and begin with providing two auxiliary lemmas.

\begin{lemma}
\label{l:aux1}
The property of conditional acyclicity of TCP-nets is hereditary. That is, given two TCP-nets $\cN = \langle \bV,{\sf cp},{\sf i},{\sf
  ci},{\sf cpt},{\sf cit} \rangle$ and $\cN' = \langle \bV',{\sf cp}',{\sf i}',{\sf  ci}',{\sf cpt}',{\sf cit}' \rangle$, if
\begin{enumerate}
  \item $\cN$ is conditionally acyclic, and
  \item $\bV' \subseteq \bV$, ${\sf cp}'\subseteq {\sf cp}$, ${\sf i}'\subseteq {\sf i}$, ${\sf ci}'\subseteq {\sf ci}$, ${\sf cpt}'\subseteq {\sf cpt}$, ${\sf cit}'\subseteq {\sf cit}$,
\end{enumerate}
then  $\cN'$ is also conditionally acyclic.
\end{lemma}

\begin{proof}
  The proof is straightforward from Definition~\ref{d:adag} since removing nodes and/or edges from $\cN$, as well as removing some preference and importance information from CPTs and CITs of $\cN$, can only remove cycles from the $\bw$-directed graphs of $\cNN$. Hence, if $\cN$ is conditionally acyclic, then so is any subnet of $\cN$.
\end{proof}

\begin{lemma}
\label{l:aux2}
Every conditionally acyclic TCP-net $\cN = \langle \bV,{\sf cp},{\sf i},{\sf ci},{\sf cpt},{\sf cit} \rangle$ contains at least one variable $X \in \bV$, such that, for each $Y \in \bV \setminus \{X\}$, we have $\cparc{Y}{X} \not\in {\sf cp}$, $\iarc{Y}{X} \not\in {\sf i}$, and $\ciarc{X}{Y} \not\in {\sf ci}$.
\end{lemma}

\begin{proof}
To prove the existence of such a root variable $X \in \cN$,
consider the dependency graph $\cNN$. Since $\cN$ is
conditionally acyclic, there has to be a node $X' \in \cNN$ that
has neither incoming directed nor undirected edges associated
with it. The see the latter, observe that (i)
every endpoint of an undirected edge in $\cNN$ will also have an
incoming directed edge, and (ii) there has to be at least one
node in $\cNN$ with no incoming directed edges, or otherwise the
conditional acyclicity of $\cN$ will be trivially violated.
However, such a node $X'$ will also be a root node in $\cN$ since
the edge set of $\cNN$ is a superset of that of $\cN$.
\end{proof}

\begin{theorem}
\label{l:sat1}
Every conditionally acyclic \TCPn\ is satisfiable.
\end{theorem}

\begin{proof}
We prove this constructively by building a satisfying preference
ordering. In fact, our inductive hypothesis will be stronger: any
conditionally acyclic \TCPn\ has a strict total order that
satisfies it. The proof is by induction on the number of problem
variables. The result trivially holds for one variable by
definition of CPTs, since we can simply use any strict total
order consistent with its CPT (and trivially satisfying
Definition~\ref{def:tcp-satisf}.)

Assume that the theorem holds for all conditionally acyclic
\TCPns\ with fewer than $n$ variables. Let $\cN$ be a \TCPn\ over
$n$ variables, and $X$ be one of the root variables of $\cN$.
(The existence of such a root $X$ is guaranteed by
Lemma~\ref{l:aux2}.) Let $\cD(X) = \{x_1, \ldots, x_k\}$ be the
domain of the chosen root variable $X$, and let $x_1 \prec \ldots
\prec x_k$ be a total ordering of $\cD(X)$ that is consistent
with the (possibly partial) preferential ordering dictated by
$CPT(X)$ in $\cN$. For each $x_i$, $1 \leq i \leq k$, construct a
\TCPn\ $\cN_i$, with $n-1$ variables $\bV - \{X\}$ by removing
$X$ from the original network, and:
\begin{enumerate}
\item For each variable $Y$, such that there is a ${\mathsf{cp}}$-arc
  $\cparc{X}{Y} \in \cN$, revise the CPT of $Y$ by restricting each row
  to $X=x_i$.
\item For each ${\mathsf{ci}}$-arc $\gamma = \ciarc{Y_1}{Y_2}$, such
  that $X \in \cS(\gamma)$, revise the CIT of $\gamma$ by restricting
  each row to $X=x_i$. If, as a result of this restriction, all rows
  in the new CIT express the same relative importance between $Y_1$
  and $Y_2$, replace $\gamma$ in $\cN_i$ by the corresponding
  ${\mathsf{i}}$-arc, i.e., either $\iarc{Y_1}{Y_2}$ or
  $\iarc{Y_2}{Y_1}$. Alternatively, if the CIT of $\gamma$ becomes
  empty, then $\gamma$ is simply removed from $\cN_i$.
\item Remove the variable $X$, together with all ${\mathsf{cp}}$-arcs
  of the form $\cparc{X}{Y}$, and all ${\mathsf{i}}$-arcs of the form
  $\iarc{X}{Y}$.
\end{enumerate}

From Lemma~\ref{l:aux1} we have that conditional acyclicity of
$\cN$ implies conditional acyclicity of all the reduced TCP-nets
$\cN_i$. Therefore, by the inductive hypothesis we can construct
a preference ordering $\succ_i$ for each of the reduced networks
$\cN_i$. Now we can construct the preferential ordering for the
original network $\cN$ as follows. Every outcome with $X=x_j$ is
ranked as preferred to any outcome with $X=x_i$, for $1 \leq i <
j \leq k$. All the outcomes with identical $X$ value, $x_i$, are
ranked according to the ordering $\succ_i$ associated with
$\cN_i$ (ignoring the value of $X$). Clearly, by construction, the
ordering we defined is a strict total order: it was obtained by
taking a set of strict total orders and ordering them,
respectively. From Definition~\ref{def:tcp-satisf}, it is easy to
see that this strict total order satisfies $\cN$.
\end{proof}

A close look at the proof of Theorem~1 reveals that the key
property of conditionally acyclic TCP-nets is that they induce an
``ordering'' over the nodes of the network. This ordering is not
fixed, but is context dependent. Different assignments to the
variables in the prefix of this ordering will yield different
suffixes. Put differently, the ordering depends on the values of
the variables, and it captures the relative importance of each
variable in each particular context. In particular, nodes that
appear earlier in the ordering are more important in this
particular context.

The above observation helps explain the rationale for our
definition of the dependency graph
(Definition~\ref{def:tcp-dependency}). In some sense, this graph
captures constraints on the ordering of variables. The TCP-net is
conditionally acyclic if these constraints are satisfiable. We
use this perspective to explain some choices made in the
definition of the dependency graph which may seem arbitrary.
First, consider the direction of (unconditional) importance edges
from the more important to the less important variable. This
simply goes in line with our desire to use a topological ordering
in which the more important variables appear first. Second,
consider the direction of CP-net edges from parent to children. It
turns out that in CP-nets, there is an induced importance
relationship between parents and children: parents are more
important than their children (see~\cite{BBDHP.journal}). Thus,
edges in the dependency graph must point from parent to child.

Finally, in order to make sense of this idea of context-dependent
ordering, we must order the variables in the selector set of a
${\mathsf{ci}}$-arc before the nodes connected by this arc. The
motivation for this last choice may be a bit less clear. The
following example shows the necessity of this (i.e., why
Theorem~\ref{l:sat1} cannot be provided for a stronger notion of
TCP-net acyclicity obtained by defining $\bw$-directed graphs over
$\cN$ rather than over $\cNN$).

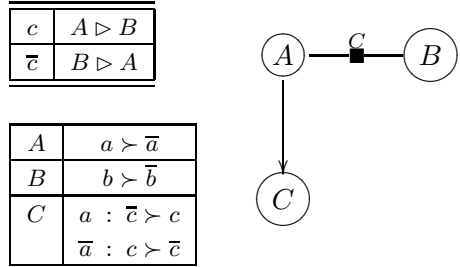
\begin{center}
\begin{minipage}{0.4\textwidth}
\begin{center}
\begin{tabular}{lr}
{\footnotesize \setlength{\extrarowheight}{2pt}
 \begin{tabular}{|c|c|}
  \hline
  \hline
   $c$ & $A \imp B$\\
  \hline
   $\overline{c}$ & $B \imp A$\\
   \hline
   \hline
  \end{tabular}
}
&
\multirow{3}{0.5\textwidth}{
    \def\objectsizestyle{\scriptstyle}
    \xymatrix @R=35pt@C=35pt{
      *++[o][F]{A} \ar@{-}[r]|{\blacksquare}^{C} \ar[d] & *++[o][F]{B}\\
       *++[o][F]{C}
              }
}\\
\begin{minipage}{1cm}
\vspace{0.5cm}
\end{minipage}
\\
{\footnotesize \setlength{\extrarowheight}{2pt}
\begin{tabular}{|c|c|}
\hline
$A$ & $a \succ \overline{a}$\\
\hline
$B$ & $b \succ \overline{b}$\\
\hline
$C$ & $a \;:\; \overline{c} \succ c$\\
  & $\overline{a} \;:\; c \succ \overline{c}$\\
\hline
\end{tabular}
}
\end{tabular}
\end{center}
\end{minipage}
\end{center}

Consider a TCP-net as depicted above. This TCP-net $\cN$ is defined over three boolean variables $\bV = \{A, B, C\}$, and having ${\sf cp} = \{ \cparc{A}{C} \}$, ${\sf ci} = \{ \ciarc{A}{B} \}$ with $\cS(A,B) = \{C\}$, and ${\sf i} = \emptyset$.
Clearly, the two possible $\bw$-directed graphs of $\cN$ (not of $\cNN$) are acyclic. Now, suppose that there exists a strict partial order $\succ'$ over $\cD(\bV)$ that satisfies $\cN$. By Definition~\ref{def:tcp-satisf}, we have
\begin{enumerate}[(1)]
\item $a\overline{b}\overline{c} \succ' a\overline{b}c $ (from $CPT(C)$),
\item $\overline{a}b\overline{c} \succ' a\overline{b}\overline{c}$ (from $CIT(\ciarc{A}{B})$ and $CPT(B)$),
\item $\overline{a}bc \succ' \overline{a}b\overline{c}$ (from $CPT(C)$), and
\item $a\overline{b}c \succ' \overline{a}bc $ (from $CIT(\ciarc{A}{B})$ and $CPT(A)$).
\end{enumerate}
However, this implies that $\succ'$ is not anti-symmetric, contradicting our assumption that $\succ'$ is a strict partial order.


\section{Verifying Conditional Acyclicity}
\label{S:tcp-consistency}

In contrast to standard acyclicity in directed graphs, the property
of conditional acyclicity cannot be easily tested in general. Naive verification of the
acyclicity of every $\bw$-directed graph can require time exponential in the
size of $\cS(\cN)$. Here we study the complexity of
verifying conditional acyclicity, discuss some hard and polynomial subclasses
of this problem, and provide some sufficient and/or necessary conditions for
conditional acyclicity that can be easily checked for certain subclasses of TCP-nets.

Let $\cN$ be a \TCPn.
If there are no cycles in the undirected graph underlying $\cNN$
(i.e., the graph obtained from $\cNN$ by making all directed edges
into undirected edges), then clearly all $\bw$-directed graphs of
$\cNN$ are acyclic, and this property of $\cNN$ is simple to check.
Alternatively, suppose that the underlying undirected graph of $\cNN$ does contain cycles. If projection of each such cycle back to $\cNN$
contains directed arcs
oriented in different directions on the cycle (one ``clockwise'' and another ``counter-clockwise''), then all $\bw$-directed graphs of $\cNN$ are still guaranteed to be acyclic. For instance, any subset (of size $>$ 2) of the variables $\{T,A,S,C\}$ in our running example in Figure~\ref{fig:dependency-flight} forms a cycle in the undirected graph underlying $\cNN$, yet each such cycle satisfies the aforementioned criterion.
This sufficient condition for conditional acyclicity can also be checked in
(low order) polynomial time.

The remaining cases are where the dependency graph $\cNN$
contains what we define below as {\em semi-directed\/} cycles, and in the rest of this section we study these cases more closely.

\begin{definition}
\label{d:sdcyc}
Let $\cA$ be a mixed set of directed and undirected edges, and $\cA^{U}$ be the undirected graph underlying $\cA$ (that is, the graph obtained from $\cA$ by dropping orientation of its directed edges.)
We say that $\cA$ is a {\em semi-directed cycle} if and only if
\begin{enumerate}[(1)]
\item $\cA^{U}$ forms a simple cycle (that is, $\cA^{U}$ consists of a single connected
component with all vertices having degree 2 w.r.t.~$\cA^{U}$).
\item Not all of the edges in $\cA$ are directed.
\item All the directed edges of $\cA$ point in the same direction along $\cA^{U}$ (i.e., ``clockwise'' or ``counter-clockwise'').
\end{enumerate}
\end{definition}

Each assignment $\bw$ to the selector sets of ${\mathsf{ci}}$-arcs in
a semi-directed cycle $\cA$ of $\cNN$
induces a direction for all these
${\mathsf{ci}}$-arcs. We say that semi-directed cycle $\cA$ is {\em
  conditionally acyclic\/} if under no such assignment $\bw$ do we
obtain a directed cycle from $\cA$. Otherwise, $\cA$ is called {\em
  conditionally directed}. Figure~\ref{fig:smcycle} illustrates a semi-directed cycle (based on the variables from our running example) with two possible configurations of its CITs that make this semi-directed cycle conditionally directed and conditionally acyclic, respectively.

\begin{figure}[t]
\begin{center}
\begin{tabular}{cc}
\begin{minipage}{2in}
{\footnotesize
 \def\objectsizestyle{\scriptstyle}
          \xymatrix @R=45pt@C=40pt{
            *++[o][F]{T} \ar[d] \ar@{-}[r]|{\blacksquare}^{D} & *++[o][F]{A}\\
            *++[o][F]{S} \ar@{-}[r]|{\blacksquare}^{D} & *++[o][F]{C} \ar[u]
            }
          }
\end{minipage} &
\begin{tabular}{c}
          {\scriptsize \setlength{\extrarowheight}{2pt}
            \begin{tabular}{|c|c|}
              \hline
              $D_{1d} $ & $S \imp C$\\
              \hline
            \end{tabular}
            }\\
            \ \\
          {\scriptsize \setlength{\extrarowheight}{2pt}
            \begin{tabular}{|c|c|}
              \hline
              $D_{1d} $ & $A \imp T$\\
              \hline
            \end{tabular}
            }\\
            {\footnotesize (a)}\\
            \ \\
          {\scriptsize \setlength{\extrarowheight}{2pt}
            \begin{tabular}{|c|c|}
              \hline
              $D_{1d} $ & $S \imp C$\\
              \hline
            \end{tabular}
            }\\
            \ \\
          {\scriptsize \setlength{\extrarowheight}{2pt}
            \begin{tabular}{|c|c|}
              \hline
              $D_{1d} $ & $T \imp A$\\
              \hline
            \end{tabular}
            }\\
            {\footnotesize (b)}
\end{tabular}
\end{tabular}
\end{center}
\caption{A semi-directed cycle: (a) conditionally directed, and (b) conditionally acyclic.}
\label{fig:smcycle}
\end{figure}
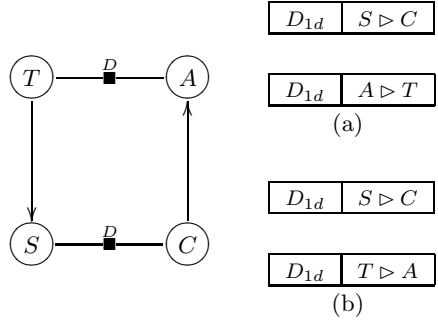

Using these notions, Lemma~\ref{l:decomposition} shows that testing conditional acyclicity
for TCP-nets is naturally decomposable.

\begin{lemma}
  \label{l:decomposition}
  A TCP-net $\cN$ is conditionally acyclic if and only if every semi-directed cycle of $\cNN$ is conditionally acyclic.
\end{lemma}

\begin{proof}
  The proof is straightforward: If there is a variable assignment that makes
  one of the semi-directed cycles of $\cNN$ conditionally directed, then no
  other cycle need be examined. Conversely, consider one of the semi-directed
  cycles $\cA$ of $\cNN$. If no assignment to $\cS(\cA)$ makes $\cA$
  conditionally directed, then additional assignments to variables in other
  selector sets do not change this property.
\end{proof}

The decomposition presented by Lemma~\ref{l:decomposition} allows us
to prove our first complexity result for testing conditional
acyclicity. Theorem~\ref{th:TCP-hard} below shows that determining
that a \TCPn\ is conditionally acyclic is {\sc coNP}-hard.

\begin{theorem}
\label{th:TCP-hard}
Given a binary-valued TCP-net $\cN$, the decision problem: is there a
conditionally directed cycle in $\cNN$, is {\sc NP}-complete, even if for every
${\mathsf{ci}}$-arc $\gamma \in \cN$ we have $|\cS(\gamma)| = 1$.
\end{theorem}

\begin{proof}
  \label{page-th:TCP-hard}
  The proof of hardness is by reduction from {\sc 3-sat}.  Given a {\sc 3-cnf}
  formula $\cF$, construct the following TCP-net $\cN$. For every variable
  $X_i$ and every clause $C_j$ in $\cF$, construct a boolean variable $X_i$
  and variable $C_j$ in $\cN$, respectively (we retain the same names, for
  simplicity). In addition, for every clause $C_j$, construct three boolean
  variables $L_{j,k}$, $1 \leq k \leq 3$, corresponding to the literals
  appearing in $C_j$.  Let $n$ be the number of clauses in $\cF$.  The TCP-net
  $\cN$ is somewhat degenerate, since it has no ${\mathsf{cp}}$-arcs. However,
  it has an ${\mathsf{i}}$-arc $\iarc{C_j}{L_{j,k}}$ for each clause $C_j$ and
  every literal $L_{j,k} \in C_j$.  In addition, for every literal $L_{j,k} \in
  C_j$, there is a ${\mathsf{ci}}$-arc $\ciarc{L_{j,k}}{C_{(j+1) \mod n}}$,
  whose selector variable is the variable $X_i$ represented in $L_{j,k}$. The
  relative importance between $L_{j,k}$ and $C_{(j+1) \mod n}$ on the selector
  $X_i$ as follows: if $L_{j,k}$ is a positive literal,
  then variable $L_{j,k}$ is more important than $C_{(j+1)
    \mod n}$ if $X_i$ is true, and less important if $X_i$ is false. For
  negative literals, the dependence on the selector variable is reversed. This
  completes the construction - clearly a polynomial-time operation.
  Figure~\ref{fig:th:TCP-hard} illustrates the subnet of $\cN$ corresponding to
  a clause $C_j = (x_1 \vee x_2 \vee \overline{x_3})$, where $L_{j,1},
  L_{j,2},L_{j,3}$ correspond to $x_1,x_2,\overline{x_3}$, respectively.

  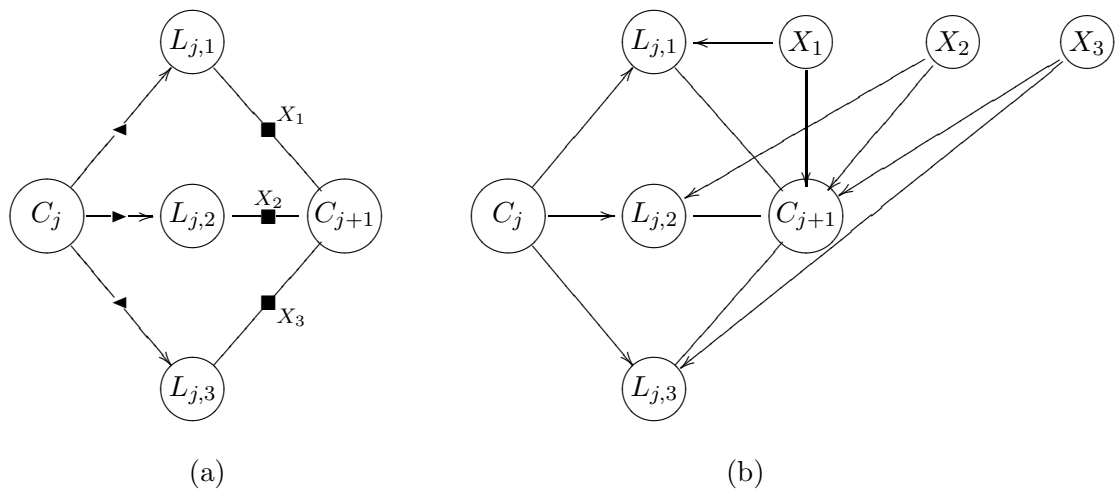
\begin{figure}[h]
   \begin{tabular}{cc}
   \begin{minipage}{2.3in}
   \mbox{
      \def\objectsizestyle{\scriptstyle}
        \xymatrix @R=40pt@C=25pt{
          & *++[o][F]{L_{j,1}} \ar@{-}[dr]|{\blacksquare}^{X_1}\\
          *+++[o][F]{C_j} \ar[ur]|{\LHD} \ar[r]|{\RHD} \ar[dr]|{\LHD} &
          *++[o][F]{L_{j,2}} \ar@{-}[r]|{\blacksquare}^{X_2} & *++[o][F]{C_{j+1}}\\
          & *++[o][F]{L_{j,3}} \ar@{-}[ur]|{\blacksquare}_{X_3}
        }
      }
   \end{minipage}
   &
   \begin{minipage}{3in}
    \mbox{
        \xymatrix @R=40pt@C=25pt{
          & *++[o][F]{L_{j,1}} \ar@{-}[dr] & *++[o][F]{X_1} \ar[l] \ar[d] & *++[o][F]{X_2} \ar[lld] \ar[ld] & *++[o][F]{X_3} \ar[lld] \ar[llldd]  \\
          *+++[o][F]{C_j} \ar[ur] \ar[r] \ar[dr] &
          *++[o][F]{L_{j,2}} \ar@{-}[r] & *++[o][F]{C_{j+1}}\\
          & *++[o][F]{L_{j,3}} \ar@{-}[ur]
        }
      }
   \end{minipage}\\
   \ \\
   (a) & (b)
   \end{tabular}
%
    \caption{(a) TCP-net subnet for $C_j = (x_1 \vee x_2 \vee \overline{x_3})$, and (b) its dependency graph.}
    \label{fig:th:TCP-hard}
  \end{figure}

  We claim that $\cNN$, the dependency graph for the network $\cN$ we just
  constructed, has a conditionally directed cycle just when
  $\cF$ is satisfiable\footnote{In this particular construction, the
    directed edges in $\cNN$ outgoing from the selector variables
    $X_i$ have no effect on the existence of
    conditionally directed cycles in $\cNN$. Therefore, here we can
    simply consider the TCP-net $\cN$ instead of its dependency graph
    $\cNN$.}. It is easy to see that there is a path from $C_j$ to
  $C_{(j+1) \mod n}$ just when the values of the variables
  participating in $C_j$ are such that $C_j$ is satisfied. Thus, an
  assignment that creates a directed path from $C_0$ to $C_0$ is an
  assignment that satisfies all clauses, and the problems are
  equivalent - hence our decision problem is {\sc NP}-hard.  Deciding
  existence of a conditional directed cycle is in {\sc NP}: Indeed,
  verifying the existence of a semi-directed cycle $\cA$ given an assignment to $\cS(\cA)$ (the union
  of the selector sets of all ${\mathsf{ci}}$-arcs in $\cA$) can
  be done in polynomial time. Thus, the problem is {\sc
    NP}-complete.
\end{proof}



One reason for the
complexity of the general problem, as emerges from the proof of
Theorem \ref{th:TCP-hard}, is the possibility that the number of
semi-directed cycles in the TCP-net dependency graph is
exponential in the size of the network. For example, the network in the reduction has
$3^n$ semi-directed cycles, due to the three possible paths generated in each
subnet as depicted in Figure~\ref{fig:th:TCP-hard}(a).
Thus, it is natural to consider networks for which the number of
semi-directed cycles is not too large. In what follows, we call a
TCP-net $\cN$ {\em $m$-cycle bounded} if the number of different
semi-directed cycles in its dependency graph $\cNN$ is at
most $m$.

From Lemma~\ref{l:decomposition} it follows that, given an $m$-cycle
bounded TCP-net $\cN$, if $m$ is polynomial in the size of $\cN$, then
we can reduce testing conditional acyclicity of $\cNN$ into separate tests
for conditional acyclicity of every semi-directed
cycle $\cA$ of $\cNN$. A
naive check for the conditional acyclicity of a semi-directed cycle
$\cA$ requires time exponential in the size of $\cS(\cA)$ -- where
$\cS(\cA)$ is the union of the selector sets of all ${\mathsf{ci}}$-arcs
in $\cA$.
Thus, if $\cS(\cA)$ is small for each
semi-directed cycle in $\cNN$, then conditional
acyclicity of $\cNN$ can be checked quickly.  In fact, often we can
determine that a semi-directed cycle $\cA$ is conditionally directed/acyclic
even more efficiently than enumerating all possible assignments to $\cS(\cA)$.

\begin{lemma}
\label{l:cond1}
Let $\cA$ be a semi-directed cycle in $\cNN$. If $\cA$ is
conditionally acyclic, then it contains a pair of
${\mathsf{ci}}$-arcs $\gamma_i,\gamma_j$ such that
$\cS(\gamma_i)\cap\cS(\gamma_j)\neq\emptyset$.
\end{lemma}

In other words, if the selector sets of the ${\mathsf{ci}}$-arcs in $\cA$
are all pairwise disjoint, then $\cA$ is conditionally directed. Thus,
Lemma~\ref{l:cond1} provides a necessary condition for conditional
acyclicity of $\cA$ that can be checked in time polynomial in the number of
variables.

\begin{nproof}{Lemma~\ref{l:cond1}}
 If all selector sets of the ${\mathsf{ci}}$-arcs in $\cA$ are pairwise disjoint, then trivially
 there exists an assignment to $\cS(\cA)$ orienting all the ${\mathsf{ci}}$-arcs of $\cA$ in one direction.
\end{nproof}

Before developing sufficient conditions for conditional
acyclicity, let us introduce some useful notation. First, given a
${\mathsf{ci}}$-arc $\gamma = \ciarc{X}{Y}$, we say that an
assignment $\bw$ to a subset $\cS'$ of $\cS(\gamma)$ {\em
orients} $\gamma$ if all rows in $\CIT(\gamma)$ consistent with
$\bw$  express the same relative importance between $X$ and $Y$,
if any. In other words, $\bw$ orients $\gamma$ if, given $\bw$,
the relative importance between $X$ and $Y$ is independent of
$\cS(\gamma) \setminus \cS'$.  Second, if a semi-directed cycle
$\cA$ contains some directed edges, we refer to their (by definition, unique) direction as the {\em direction of $\cA$}.

\begin{lemma}
  \label{l:cond2}
  A semi-directed cycle $\cA$ is conditionally acyclic if it contains a pair of
  ${\mathsf{ci}}$-arcs $\gamma_i,\gamma_j$ such that either:

  \begin{enumerate}[(a)]
  \item $\cA$ contains directed edges, and for every assignment $\bw$
    to $\cS(\gamma_i)\cap\cS(\gamma_j)$, either $\gamma_i$ or $\gamma_j$ is
    oriented by $\bw$ in the direction opposite to the direction of $\cA$.

  \item All edges in $\cA$ are undirected, and for every assignment
    $\bw$ to $\cS(\gamma_i)\cap\cS(\gamma_j)$, $\gamma_i$ and
    $\gamma_j$ are oriented by $\bw$ in opposite directions with
    respect to $\cA$.
  \end{enumerate}
\end{lemma}

\begin{proof} Follows immediately from the conditions in the lemma.
\end{proof}

Lemma~\ref{l:cond2} provides a sufficient condition for conditional acyclicity
of $\cA$ that can be checked in time exponential in the maximal size of
selector set intersection for a pair of ${\mathsf{ci}}$-arcs in $\cA$.
Note that the size of the \tcpn\ is at least as large as the above
exponential term, because the
description of the CIT is exponential in the size of the corresponding
selector set. Thus, checking this condition
is only linear in the size of the network.

\begin{definition}
  Given a semi-directed cycle $\cA$, let $\shared(\cA)$ denote
  the union of all pairwise intersections of the selector sets of the
  ${\mathsf{ci}}$-arcs in $\cA$:
  \[
    \shared(\cA) = \bigcup_{\gamma_i,\gamma_j\in\cA}{\cS(\gamma_i)\cap\cS(\gamma_j)}
  \]
\end{definition}

\begin{lemma}
  \label{l:cond3}
\

\begin{enumerate}[(a)]
\item If a semi-directed cycle
  $\cA$ contains directed edges, then $\cA$ is conditionally
  acyclic if and only if, for each assignment $\bu$ on $\shared(\cA)$,
  there exists a ${\mathsf{ci}}$-arc $\gamma_{\bu} \in \cA$ that is
  oriented by $\bu$ in the direction opposite to the direction of
  $\cA$.

\item If a semi-directed cycle $\cA$ contains only
${\mathsf{ci}}$-arcs, then $\cA$ is conditionally acyclic if and only
if, for each assignment $\bu$ on $\shared(\cA)$, there exist two
${\mathsf{ci}}$-arcs $\gamma_{\bu}^1, \gamma_{\bu}^2\in \cA$ that are
oriented by $\bu$ in opposite directions with respect to $\cA$.
\end{enumerate}
\end{lemma}

\begin{proof}
  The sufficiency of the above condition is clear, since it subsumes
  the condition in Lemma~\ref{l:cond2}. Thus, we are left with proving
  necessity.
  We start with the second case in which $\cA$ contains only
  ${\mathsf{ci}}$-arcs. Assume to the contrary that $\cA$ is
  conditionally acyclic, but there exists an assignment $\bu$ on
  $\shared(\cA)$ such that no pair of  ${\mathsf{ci}}$-arcs in $\cA$ are
  oriented by $\bu$ in opposite directions with respect to $\cA$.

  For each ${\mathsf{ci}}$-arc $\gamma \in \cA$, let $\cS^*(\gamma) =
  \cS(\gamma) \setminus \shared(\cA)$. Consider the following disjoint
  partition $\cA = \cA^{{\mathsf{i}}}_{\bu} \cup
  \cA^{{\mathsf{ci}}}_{\bu}$ induced by $\bu$ on $\cA$: For each
  ${\mathsf{ci}}$-arc $\gamma \in \cA$, if $\bu$ orients $\gamma$,
  then we have $\gamma \in \cA^{{\mathsf{i}}}_{\bu}$. Otherwise, if
  the direction of $\gamma$ is not independent of $\cS^*(\gamma)$
  given $\bu$, we have $\gamma \in \cA^{{\mathsf{ci}}}_{\bu}$. We make
  two observations:
  \begin{enumerate}
  \item \label{l:observ1} Our initial (contradicting) assumption
    implies that all the (now directed) edges in
    $\cA^{{\mathsf{i}}}_{\bu}$ agree on the direction with respect to
    $\cA$.
  \item \label{l:observ2} If for some ${\mathsf{ci}}$-arc $\gamma \in
    \cA$ we have $\cS^*(\gamma) = \emptyset$, then we have
    $\gamma \in \cA^{{\mathsf{i}}}_{\bu}$, since all the
    selectors of $\gamma$ are instantiated by $\bu$.
  \end{enumerate}
  If we have $\cA^{{\mathsf{ci}}}_{\bu} = \emptyset$, then the first
  observation trivially contradicts our initial assumption that $\cA$
  is conditionally acyclic. Alternatively, if
  $\cA^{{\mathsf{ci}}}_{\bu} \neq \emptyset$, then, by definition of
  $\shared(\cA)$, we have that $\cS^*(\gamma_i) \cap \cS^*(\gamma_j) =
  \emptyset$ for each pair of ${\mathsf{ci}}$-arcs $\gamma_i,\gamma_j
  \in \cA^{{\mathsf{ci}}}_{\bu}$. This means that we
  can assign each such (non-empty, by the second observation)
  $\cS^*(\gamma_i)$ independently, and thus can extend $\bu$ into an
  assignment on $\cS(\cA)$ that will orient all the
  ${\mathsf{ci}}$-arcs in $\cA^{{\mathsf{ci}}}_{\bu}$ either in the
  direction of $\cA^{{\mathsf{i}}}_{\bu}$ if $\cA^{{\mathsf{i}}}_{\bu}
  \neq \emptyset$, or in an arbitrary joint direction if
  $\cA^{{\mathsf{i}}}_{\bu} = \emptyset$. Again, this contradicts our
  assumption that $\cA$ is conditionally acyclic. Hence, we
  have proved that our condition is necessary for the second
  case. The proof for the first case in which $\cA$ contains some
  directed edges is similar.
\end{proof}

\commentout{
 \begin{proof}
   The sufficiency of the above condition is clear, since it subsumes
   the condition in Lemma~\ref{l:cond2}.  Now, in order to prove its
   necessity, assume to the contrary, that there exists a
   semi-directed cycle $\cA$ which is not a conditionally directed
   cycle, and there exist an assignment $\pi$ on $\shared(\cA)$ such
   that no ${\mathsf{ci}}$-arc can be converted into an
   ${\mathsf{i}}$-arc that violates the direction of $\cA$. By using
   the direction of $\cA$ we assumed that $\cA$ contains some
   ${\mathsf{cp}}$ or ${\mathsf{i}}$ arcs; the proof for the case that
   $\cA$ consists only of ${\mathsf{ci}}$-arcs is be similar.

   Now we show that the condition provided by the lemma is necessary
   for conditional acyclicity of $\cA$. For $1 \leq i\leq k$, let
   $\cS'(\gamma_i) = \cS(\gamma_i) \setminus \shared(\cA)$. By
   definition of $\shared(\cA)$, for each pair of
   ${\mathsf{ci}}$-arcs $\gamma_i,\gamma_j \in \cA$ we have
   $\cS'(\gamma_i) \cap \cS'(\gamma_j) = \emptyset$.  This means that
   we can assign each $\cS'(\gamma_i)$ independently, and thus can
   extend $\pi$ into an assignment on $\cS(\cA)$ that will orient all
   the ${\mathsf{ci}}$-arcs in $\cA$ according to the direction of
   $\cA$.  Clearly, this contradicts our assumption that $\cA$ is
   conditionally acyclic.
\end{proof}
}

In general, the size of $\shared(\cA)$ is $O(|\bV|)$. Since we have to
check the set of assignments over $\shared(\cA)$, this implies that
the problem may be hard. Theorem~\ref{th:1-cycle-TCP-hard} shows that
this is indeed the case.


\begin{theorem}
\label{th:1-cycle-TCP-hard}
Given a binary-valued, $1$-cycle bounded TCP-net $\cN$, the decision problem:
is there a conditionally directed cycle in $\cNN$, is {\sc NP}-complete, even if
for every ${\mathsf{ci}}$-arc $\gamma \in \cN$ we have $|\cS(\gamma)| \leq 3$.
\end{theorem}

\begin{proof}
  \label{page-th:1-cycle-TCP-hard}
  The proof of hardness is by reduction from {\sc 3-sat}.  Given a {\sc 3-cnf}
  formula $\cF$, construct the following TCP-net $\cN$. For every variable
  $X_i$ and every clause $C_j$ in $\cF$, construct boolean variables $X_i$
  and $C_j$ in $\cN$, respectively.
  In addition, add a single dummy variable $C$, and an ${\mathsf{i}}$-arc
  $\iarc{C}{C_1}$.  Let $n$ be the number of clauses in $\cF$. For $1 \leq j \leq n-1$,
  we have $n-1$ ${\mathsf{ci}}$-arcs $E_j = \ciarc{C_j}{C_{j+1}}$. In addition, we have
  ${\mathsf{ci}}$-arc $E_n = \ciarc{C_n}{C}$.
  For all $1 \leq j \leq n$, the $CIT$ for $E_j$ is
  determined by clause $C_j$, as follows. The selector set for $E_j$ is just
  the set of variables appearing in $C_j$, and the relative importance between
  the variables of $E_j$ is determined as follows: $C_j$ is less important than
  $C_{j+1}$ just when the values of the variables in the selector set are such
  that $C_j$ is false. (For $j=n$, read $C$ instead of $C_{j+1}$).

  The constructed TCP-net $\cN$ is $1$-cycle bounded, because there is only one semi-directed
  cycle in its dependency graph $\cNN$, namely $C, C_1,\ldots, C_n, C$.
  We claim that this semi-directed cycle is conditionally directed
  just when $\cF$
  is satisfiable. It is easy to see that the directed path from $C$ to $C$
  exists when all the ${\mathsf{ci}}$-arcs are being directed from $C_j$ to
  $C_{j+1}$, which occurs exactly when the variable assignment makes the clause
  $C_j$ satisfiable.  Hence, a directed cycle occurs in $\cN$ exactly when the
  assignment makes all clauses satisfiable, making the two problems
  equivalent. Thus our decision problem is {\sc NP}-hard.  Finally, as
  deciding existence of a conditional directed cycle is in {\sc NP} (see the
  proof of Theorem~\ref{th:1-cycle-TCP-hard}), the problem is {\sc
    NP}-complete.
\end{proof}


Observe that the proof of Theorem~\ref{th:1-cycle-TCP-hard} does not work when the size of all the selector sets is bounded by
$2$, because {\sc 2-sat} is in {\sc P}.  The immediate question is whether in this latter
case the problem becomes tractable, and for binary-valued TCP-nets the answer is affirmative.

\begin{theorem}
\label{th:m-cycle-TCP-easy}
Given a binary-valued, $m$-cycle bounded TCP-net $\cN$, where $m$ is polynomial
in the size of $\cN$ and, for every ${\mathsf{ci}}$-arc $\gamma \in \cN$ we
have $|\cS(\gamma)| \leq 2$, the decision problem: is there a conditional
directed cycle in $\cNN$, is in {\sc P}.
\end{theorem}

\begin{proof}
 \label{page-th:m-cycle-TCP-easy} The proof uses a reduction from conditional acyclicity
 testing to satisfiability. Let $\cA$ be a semi-directed cycle
 with $|\cS(\gamma)| \leq k$ for every ${\mathsf{ci}}$-arc $\gamma \in \cA$. We reduce the
 conditional acyclicity  testing problem
 to an equivalent {\sc k-sat} problem instance. In particular,
 since {\sc 2-sat} is  solvable in linear time~\cite{Even:Itai:Shamir:76},
 together with Lemma~\ref{l:decomposition} this proves the claim.

 First, assume that $\cA$ has at least one directed edge (either
 ${\mathsf{i}}$-arc or ${\mathsf{cp}}$-arc). By definition of semi-directed
 cycles, all directed edges of $\cA$ point in the same direction,
 specifying the only possible cyclic orientation $\omega$ of $\cA$.
 For each ${\mathsf{ci}}$-arc $\gamma_i \in \cA$, let the selector set be
 $\cS(\gamma_i) = \{X_{i,1},..., X_{i,k}\}$.%
\footnote{
 If $|\cS(\gamma_i)| < k$, the only impact will be a more compact
reduction below.}
 Clearly, $\cA$ is conditionally directed
 if and only if all the ${\mathsf{ci}}$-arcs of $\cA$ can be directed
 consistently with $\omega$.

 Given such a semi-directed cycle $\cA$, we create a corresponding {\sc k-cnf}
 formula $\cF$, such that $\cF$ is satisfiable just when $\cA$ is conditionally
 directed.
 Let us call all $CIT(\gamma_i)$ entries that are consistent with $\omega$
 by the term {\em $\omega$-entries}.
 Since $\cS(\gamma) = \{X_{i,1},... ,X_{i,k}\}$ and $\cN$ is binary valued,
 we can represent the non-$\omega$ entries in $CIT(\gamma_i)$ as a conjunction
 of disjunctions, i.e., in CNF form. The number of disjunctions is equal
 to the number of non-$\omega$ entries in $CIT(\gamma_i)$, and each disjunction
 is comprised of $k$ literals.
Thus, the representation of $CIT(\gamma_i)$
 is a $k$-CNF formula, of size linear in the size of $CIT(\gamma_i)$.
(In fact, the size of the
resulting formula can sometimes be significantly smaller, as one can frequently
simplify the component CNF fragments, but this property is not needed here.)

 Finally, compose all the CNF representations of the $CIT(\gamma_i)$, for
 every $\gamma_i \in \cA$, resulting in a $k$-CNF formula of size linear in
 the combined number of table entries.
 The construction of $\cF$ is clearly a linear-time operation.
 Likewise, it is easy to see that $\cF$ is satisfiable just when there
 is an assignment to $\cS(\cA)$ converting $\cA$ into a directed cycle.

The minor
unresolved issue is with semi-directed cycles consisting of
${\mathsf{ci}}$-arcs only. Given such a semi-directed cycle $\cA$, we reduce
the problem into two sub-problems with a directed arc. Let
$\cA^\prime$ and $\cA^{\prime\prime}$ be cycles created from $\cA$
by inserting one dummy variable and one ${\mathsf{i}}$-arc
into $\cA$ -- clockwise for $\cA^\prime$,
counter-clockwise for $\cA^{\prime\prime}$.
Now, $\cA$ is conditionally directed
if and only if either $\cA^\prime$ or $\cA^{\prime\prime}$ (or both) are conditionally directed.
\end{proof}


To summarize our analysis of verifying conditional acyclicity, one must first identify the semi-directed
cycles in the dependency graph of the TCP-net.
Next, one must show that given each assignment $\bw$ to the
importance-conditioning variables of each semi-directed cycle, the $\bw$-directed graph
is acyclic. This problem is {\sc coNP}-hard in general networks%
\footnote{This actually means that when the network is not too large,
we can probably solve this in a reasonable amount of time.},
but there are interesting classes of networks in which it is tractable.
This is the case when the number of semi-directed cycles
is not too large and either the size of $\shared(\cA)$ for each such cycle
or the size of each selector set is not too large.
Note that in practice, one would expect to have small selector sets --
statements such as ``$X$ is more important than $Y$ when $A=a$ and $B=b$ and $\ldots$ and $Z=z$''
appear to be more complex than what one would expect to hear.
Thus, Lemma~\ref{l:cond2},
Lemma~\ref{l:cond3} (for semi-directed cycles with small $\shared(\cA)$), and
Theorem~\ref{th:m-cycle-TCP-easy} are of more than just
theoretical interest.
Naturally, extending the toolbox of TCP-net
subclasses that can be efficiently tested for consistency is clearly
of both theoretical and practical interest.

\section{Reasoning about Conditionally Acyclic TCP-nets}
\label{S:constraints}
While automated consistency verification is the core part of the
preference elicitation stage, efficiency of reasoning about user
preferences is one of the main desiderata of any model for
preference representation. Of particular importance is the task of
preference-based optimization and constrained optimization, which we
discuss in the first part of this section. Another important task, which
provides an important component in the algorithm for
constrained optimization
we present, is outcome comparison -- discussed in the second part of
this section.

\subsection{Generating Optimal Assignments}
Following the notation of~\citeA{BBDHP.journal},
if $\bx$ and $\by$ are assignments to disjoint
subsets $\bX$ and $\bY$ of the variable set $\bV$, respectively,
we denote the combination of $\bx$ and $\by$ by $\bx\by$.
If  $\bX\cap \bY = \emptyset$ and $\bX\cup \bY = \bV$,
we call $\bx\by$ a {\em completion} of assignment $\bx$, and
denote by $\cmplet(\bx)$ the set of all completions of $\bx$.

One of the central properties of the original CP-net
model~\cite{BBDHP.journal} is that,
given an acyclic CP-net $\cN$ and a
(possibly empty) partial assignment $\bx$ on its variables, it is
simple to determine an outcome consistent with $\bx$ (a completion of $\bx$)
that is {\em preferentially optimal} with respect to
$\cN$. The corresponding linear time {\em forward sweep} procedure is
as follows: Traverse the variables in some topological order induced
by $\cN$, and set each unassigned variable to its most preferred value
given its parents' values.

Our immediate observation is that this procedure works {\em as
is} also for conditionally acyclic \TCPns: The relative
importance relations do not play a role in this case, and the
network is traversed according to a topological order induced by
the CP-net part of the given \TCPn. In fact,
Corollary~\ref{cor:tcp-opt} holds for any TCP-net that has no
directed cycles consisting only of ${\mathsf{cp}}$-arcs.

\begin{corollary}
  \label{cor:tcp-opt}
  Given a conditionally acyclic TCP-net and a (possibly empty) partial assignment $\bx$ on its variables,
  the forward sweep procedure constructs the most preferred outcome in
  $\cmplet(\bx)$.
\end{corollary}

\commentout{
However, as in CP-nets, adding hard constraints on the variables
of the TCP-net makes the optimization process hard.
In the rest of this section we discuss constraint-based outcome optimization
with respect to conditionally acyclic TCP-nets.
}

This strong computational property of outcome optimization with respect to acyclic
CP-nets (and conditionally acyclic \TCPns) does not hold if some
of the \TCPn\ variables are constrained by a set of hard
constraints, $\cC$.  In this case, determining the set of
{\em preferentially non-dominated}%
\footnote{An outcome $o$ is said to be non-dominated
with respect to some
  preference order $\succ$ and a set of outcomes $S$ if there is no other $o' \in S$ such that $o'\succ o$.}  {\em feasible} outcomes is not trivial.  For acyclic
CP-nets, a branch and bound algorithm for determining the optimal feasible
outcomes was introduced by~\citeA{BBDHP.journal2}.  This algorithm has the important
{\em anytime} property -- once an outcome is added to the current set of
non-dominated outcomes, it is never removed. An important implication of this
property is that the {\em first\/} generated assignment that satisfies the set
of hard constraints is also preferentially non-dominated. In other words,  finding just one non-dominated solution
in this algorithm boils down to solving the underlying CSP under  certain variable and value ordering strategies.


Here we develop an extension/modification of the algorithm of~\citeA{BBDHP.journal2} to conditionally acyclic \TCPns. The extended
algorithm \searchtcp\ retains the anytime property and is shown in
Figure~\ref{opt2}. The key difference between processing an acyclic
CP-net and a conditionally acyclic \TCPn\ is that the semantics of the
former implicitly induces a single partial order of importance over
the variables (where each node precedes its
descendants)~\cite{BBDHP.journal}, while the semantics of the latter
induces a hierarchically-structured {\em set} of such partial orders:
Each such partial order corresponds to a single assignment to the
set of selector variables of the network, or, more specifically, to a certain $\bw$-directed graph.

\begin{figure}[t]
\begin{center}
\fbox{
  \begin{minipage}{\textwidth}
\begin{tabbing}
xx\=xx\=xx\=xx\=xx\=xx\= \kill
{\small \underline{\searchtcp}} ($\cN$, $\cC$, $\cK$)\\
${\mathsf{Input}}$: Conditionally acyclic \TCPn\ $\cN$,\\
\hspace{27pt}  Hard constraints $\cC$ on the variables of $\cN$,\\
\hspace{27pt}  Assignment $\cK$ to the variables of $\cN_{orig} \setminus \cN$.\\
${\mathsf{Output}}$: Set of all, non-dominated w.r.t.\ $\cN$, solutions for $\cC$.\\
\ \\
1. Choose any variable $X$ s.t. there is no ${\mathsf{cp}}$-arc $\cparc{Y}{X}$,\\
\>\> no ${\mathsf{i}}$-arc $\iarc{Y}{X}$, and no $\ciarc{X}{Y}$ in $\cN$.\\
2. Let $x_1 \succ \ldots \succ x_k$ be a total order on $\cD(X)$ consistent with the preference\\
\>\> ordering of $\cD(X)$ by the assignment on $Pa(X)$ in $\cK$.\\
3. Initialize the set of local results by $\cR = \emptyset$\\
4. {\bf for} $(i=1;\ i\leq k;\ i++)$ {\bf do}\\
5. \>\> $X = x_i$\\
6. \>\> Strengthen the constraints $\cC$ by $X = x_i$ to obtain $\cC_i$\\
7. \>\> {\bf if} $\cC_j \subseteq \cC_i$ for some $j < i$ {\bf or} $\cC_i$
 is inconsistent {\bf then}\\
8. \>\>\> {\bf continue} with the next iteration\\
 \>\> {\bf else}\\
 9. \>\>\> Let $\cK'$ be the partial assignment induced by $X = x_i$
  and $\cC_i$\\
 10. \>\>\> $\cN_i$ = ${\mathsf{Reduce}}$ ($\cN$,$\cK'$)\\
  11. \>\>\> Let $\cN_i^1,\ldots,\cN_i^m$ be the components of $\cN_i$, 
  connected\\
   \>\>\>\>  either by the edges of $\cN_i$ or by the constraints
  $\cC_i$.\\
  12. \>\>\> {\bf for} $(j=1;\ j\leq m;\ j++)$ {\bf do}\\
   13. \>\>\>\> $\cR_i^j$ = \searchtcp ($\cN_i^j, \cC_i, \cK\cup\cK'$)\\
  14. \>\>\> {\bf if} $\cR_i^j \neq \emptyset$ for all $j \leq m$ {\bf
    then}\\
   15. \>\>\>\> {\bf foreach} $o \in \cK' \times \cR_i^1 \times \cdots
      \times \cR_i^m$ {\bf do}\\
     16 \>\>\>\>\> {\bf if} $\cK \cup o' \not\succ \cK \cup o$  holds
      \underline{for each} $o' \in \cR$ {\bf then} add $o$ to $\cR$\\
17. {\bf return} $\cR$
\end{tabbing}
\end{minipage}
}
\end{center}
\caption{The \searchtcp\ algorithm for conditionally acyclic TCP-net based constrained optimization.}
\label{opt2}
\end{figure}

Formally, the problem is defined by a conditionally acyclic
TCP-net $\cN_{orig}$, and a set of hard constraints $\cC_{orig}$,
posed on the variables of $\cN_{orig}$.  The \searchtcp\
algorithm (depicted in Figure~\ref{opt2}) is recursive, and each
recursive call receives three parameters:
\begin{enumerate}
\item A TCP-net $\cN$, which is a subnet of the original conditionally acyclic TCP-net
  $\cN_{orig}$,
\item A set $\cC$ of hard constraints among the variables of $\cN$,
  which is a subset of the original set of constraints $\cC_{orig}$
  obtained by restricting $\cC_{orig}$ to the variables of $\cN$, and
\item An assignment $\cK$ to all the variables in $\cN_{orig} - \cN$. In what follows, we refer to this assignment $\cK$ as a {\em context}.
\end{enumerate}
The initial call to \searchtcp\ is done with $\cN_{orig}$, $\cC_{orig}$,
and $\{\}$, respectively.

Basically, the \searchtcp\ algorithm starts with an empty set of solutions, and gradually extends it by adding new non-dominated solutions to $\cC_{orig}$. At each stage of the algorithm, the current set of solutions serves as a ``lower bound'' for future candidates; A new candidate at any point is compared to all solutions generated up
   to that point. If the candidate is dominated by no member of the current  solution set, then it is added into this set.

The \searchtcp\ algorithm is guided by the graphical structure of $\cN_{orig}$.  It proceeds by assigning values to the variables in a top-down manner, assuring
that outcomes are generated in an order that satisfies (i.e., consistent with) $\cN$. On a recursive call to the \searchtcp\ procedure with a \TCPn\ $\cN$, the
eliminated variable $X$ is one of the root variables of $\cN$ (line 1). Recall that, by Lemma~\ref{l:aux2},
conditional acyclicity of $\cN$ guarantees
the existence of such a root variable $X$.  The values of  $X$ are considered according to the preference ordering induced on $\cD(X)$ by the assignment provided by the context $\cK$ to $Pa(X)$ (where $Pa(X)$ is defined with respect to $\cN_{orig}$). Note that $\cK$
necessarily contains some assignment to $Pa(X)$ since $X$ is a root variable of the currently considered subnet $\cN$ of $\cN_{orig}$.
Any additional variable assignment $X=x_i$ converts the current set of constraints $\cC$ into a strictly non-weaker
constraint set $\cC_i$.  As a result of this propagation of $X = x_i$,
values for some variables (at least, the value of $X$) are fixed
automatically, and this partial assignment $\cK'$ extends the current context
$\cK$ in recursive processing of the next variable. The ${\mathsf{Reduce}}$ procedure,
presented in Figure~\ref{fig:reduce}, refines the \TCPn\ $\cN$ with respect to
$\cK'$: For each variable assigned by $\cK'$, we reduce both the CPTs and the
CITs involving this variable, and remove this variable from the network. This
reduction of the CITs may remove conditioning of relative importance between
some variables, and thus convert some ${\mathsf{ci}}$-arcs into
${\mathsf{i}}$-arcs, and/or remove some ${\mathsf{ci}}$-arcs completely.  The
main point is that, in contrast to CP-nets, for a pair of $X$ values $x_i,x_j$,
the variable elimination orderings for processing the networks $\cN_i$ and $\cN_j$, resulting from
propagating $\cC_i$ and $\cC_j$, respectively, may  {\em disagree} on the
ordering of some variables.

\begin{figure}[ht]
\begin{center}
\fbox{
  \begin{minipage}{\textwidth}
    \begin{tabbing}
      xx\=xx\=xx\=xx\=xx\=xx\= \kill
      \underline{${\mathsf{Reduce}}$} ($\cN$, $\cK'$)\\
     1.  {\bf foreach} $\{X = x_i\} \in \cK'$ {\bf do}\\
      2. \>\> {\bf foreach} ${\mathsf{cp}}$-arc $\cparc{X}{Y} \in \cN$ {\bf do}\\
      3. \>\>\> Restrict the CPT of $Y$ to the rows dictated by $X = x_i$.\\
      4. \>\> {\bf foreach} ${\mathsf{ci}}$-arc $\gamma = \ciarc{Y_1}{Y_2} \in \cN$
      s.t. $X \in \cS(\gamma)$ {\bf do}\\
      5. \>\>\> Restrict the CIT of $\gamma$ to the rows dictated by $X = x_i$.\\
      6. \>\>\> {\bf if}, given the restricted CIT of $\gamma$, relative importance\\
      \>\>\>\>between $Y_1$ and $Y_2$ is independent of $\cS(\gamma)$, {\bf then}\\
      7. \>\>\>\>\> {\bf if} CIT of $\gamma$ is not empty {\bf then}\\
      8. \>\>\>\>\>\> Replace $\gamma$ by the corresponding ${\mathsf{i}}$-arc.\\
      9. \>\>\>\>\> {\bf else} Remove $\gamma$.\\
      10. \>\> Remove from $\cN$ all the edges involving $X$.\\
      11. {\bf return} $\cN$.
    \end{tabbing}
  \end{minipage}
  }
\end{center}
\caption{The {\sf Reduce} procedure.} \label{fig:reduce}
\end{figure}

If the partial assignment $\cK'$ causes the current CP-net to become
disconnected with respect to both the edges of the network and the
inter-variable hard constraints, then each connected component invokes an
independent search (lines 11-16). This is because optimization of the variables
within such a component is independent of the variables outside that
component. In addition, after strengthening the set of constraints
$\cC$ by $X = x_i$ to $\cC_i$ (line 6), some pruning takes place in the
search tree (lines 7-8): If the
set of constraints $\cC_i$ is strictly more restrictive than some
other set of constraints $\cC_j = \cC \cup \{X=x_j\}$ where $j < i$,
then the search under $X = x_i$ is not continued. The reason for this
pruning is that it can be shown that any feasible outcome $a$
involving $X = x_i$ is dominated by (i.e., less preferable than) some
feasible outcome $b$ involving $X = x_j$ and thus $a$ cannot be in
the set of non-dominated solutions for the original set of
constraints\footnote{ This pruning
was introduced by~\citeA{BBDHP.journal2} for acyclic CP-nets, and it remains
valid the same way for conditionally acyclic TCP-nets. For the proof of soundness of this pruning technique we refer the reader to Lemma 2 in~\cite{BBDHP.journal2}.}.  Therefore, the search is depth-first branch-and-bound,
where the set of non-dominated solutions generated so far is a proper
subset of the required set of all the non-dominated solutions for the
problem, and thus it corresponds to the current lower bound.

When the potentially non-dominated solutions for a particular
subgraph are returned with some assignment $X = x_i$, each such
solution is compared to all non-dominated solutions involving more
preferred (in the current context $\cK$) assignments $X = x_j$, $j <
i$ (line 16). A solution with 
$X = x_i$ is added to the set of the
non-dominated solutions for the current subgraph and context if and
only if it passes this non-domination test. From the
semantics of the \TCPns, given the same context $\cK$, a solution
involving $X = x_i$ can not be preferred to a solution involving $X =
x_j$, $j < i$. Thus, the generated global set $\cR$ never shrinks.


\begin{theorem}
\label{t:search-tcp-correct}
  Given a conditionally acyclic TCP-net $\cN$ and a set of hard constraints $\cC$ over the variables of $\cN$, an outcome $o$ belongs to the set $\cR$ generated by the algorithm \searchtcp\ if and only if $o$ is consistent with $\cC$, and there is no other outcome $o'$ consistent with $\cC$ such that $\cN \models o' \succ o$.
\end{theorem}

\begin{proof}
  Let $\cR_{\cC}$ be the desired set of all the preferentially non-dominated solution to $\cC$. To prove this theorem, we should show that:
  \begin{enumerate}
    \item {\em Completeness:} No preferentially
    non-dominated solution to $\cC$ is pruned out, that is, we have $\cR \supseteq \cR_{\cC}$, and
    \item {\em Soundness:} The resulting set $\cR$ contains no preferentially dominated solution to $\cC$, that is, $\cR \subseteq \cR_{\cC}$.
  \end{enumerate}
\noindent (1) The solutions to $\cC$ are pruned by \searchtcp\ only in two places, namely at the search space pruning in lines 7-8, and at the non-dominance test step in line 16. For the first case, the correctness of the pruning technique used in lines 7-8 is given by Lemma 2 in~\cite{BBDHP.journal2}, and thus this pruning does not violate completeness of \searchtcp. For the second case, if an explicitly generated
solution $o$ is rejected due to the failure of its non-dominance test, then $o \not\in \cR_{\cC}$ is apparent since the rejection of $o$ here is based on presenting a concrete solution $o'$ such that $\cN \models o' \succ o$. Hence, we have $\cR \supseteq \cR_{\cC}$.

\ \\
\noindent (2) To show $\cR \subseteq \cR_{\cC}$ it is enough to
prove that a newly generated solution cannot dominate an existing solution, that is, if $o$ was added to the generated set of solutions after $o'$ then it is not the case that $\cN \models o\succ o'$.
The proof is by induction on the number of problem
variables. First, the claim trivially holds for any one-variable TCP-net, as the order in which the solutions are examined in line 16 coincides with the total order selected for the single variable of the network in line 2. Now, assume that the claim holds for all conditionally acyclic \TCPns\ with fewer than $n$ variables.
Let $\cN$ be a \TCPn\ over $n$ variables, $\cC$ be a set of hard constraints on these variables, and $X$ be the root variable of $\cN$ selected in line 1.
Let $\cR = \{o_{1},\ldots,o_{r}\}$ be the output of \searchtcp\ for these $\cN$ and $\cC$, where the elements of $\cR$ are numbered according to the order of their non-dominance examination in line 16.
Now, assume that there exists a pair of assignments $o_{i},o_{j} \in \cR$, such that $i < j$, yet $\cN \models o_{j} \succ o_{i}$.

First, suppose that $o_{i}$ and $o_{j}$ provide the same value to
$X$, that is $o_{i} = x_{l}o_{i}'$ and $o_{j} = x_{l}o_{j}'$, for
some $x_{l} \in \cD(X)$.  In this case, however, $o_{i}'$ and
$o_{j}'$ belong to the output of the same recursive call to
\searchtcp\ with $\cN_{l}$ and $\cC_{l}$, and thus, by our
inductive hypothesis, $o_{i}'$ and $o_{j}'$ are preferentially
incomparable. Likewise, $\cN_{l}$ is obtained in line 10 by
reducing $\cN$ with respect to $x_{l}$, and thus the variables of
$\cN_{l}$ are preferentially independent of $X$. Hence,
preferential incomparability of $o_{i}'$ and $o_{j}'$ implies
preferential incomparability of $o_{i}$ and $o_{j}$, and thus
$\cN \models o_{j} \succ o_{i}$ cannot be the case.

Alternatively, suppose that $o_{i}$ and $o_{j}$ provide two
different values to $X$, that is $o_{i} = x_{l}o_{i}'$ and $o_{j}
= x_{m}o_{j}'$, $x_{l}, x_{m} \in \cD(X)$, where $\cD(X)$ is
numbered according to the total ordering of its values selected
in line 2. Observe that, by the construction of \searchtcp, $i <
j$ trivially implies $l < m$. However, using the arguments
identical to these in the constructive proof of
Theorem~\ref{def:tcp-satisf}, there exists at least one
preference order $\succ$ of the complete assignments to the
variables of $\cN$ in which we have $o_{i} \succ o_{j}$. Hence,
it cannot be the case that $\cN \models o_{j} \succ o_{i}$, and
thus contradiction of our assumption that $\cN \models o_{j} \succ
o_{i}$ is now complete.
\end{proof}

Note that, if we are interested in getting {\em one} non-dominated solution for the given set of hard constraints (which is often the case), we can output the {\em first} feasible outcome
generated by \searchtcp. No comparisons between pairs of outcomes are required because
there is nothing to compare with the first generated
solution. However, if we are interested in getting {\em all}, or even
{\em a few} non-dominated solutions,
then the efficiency of preferential comparison between pairs of outcomes becomes an important factor
in the entire complexity of the \searchtcp\ algorithm. Hence, in the next section we consider such preferential comparisons more closely.

\commentout{
\begin{definition}
  \label{def:tcp-prefix}
  Let $\cN$ be a conditionally acyclic \TCPn, and $\cS(\cN)$ be the union of
  the selector variables in $\cN$. $\cS'(\cN) \subseteq \cS(\cN)$ is referred
  to as a {\em prefix} of $\cS(\cN)$ if and only if, for each $X\in\cS'(\cN)$,
  and for each $Y\in\cS(\cN)\setminus\cS'(\cN)$, $X$ is not reachable from $Y$
  in the dependency graph of $\cN$.
\end{definition}

Let $\Omega$ be a set of partial orders over the variables of $\cN$, that agree
on an assignment to a prefix $\cS'(\cN)$ of $\cS(\cN)$.  By
Definition~\ref{def:tcp-prefix}, all the orders in $\Omega$ agree as well on
(partial) ordering of all the variables in $\cN$, relative importance between
which is fully determined by $\cS'(\cN)$.

Observe, that any set of partial orders over the variables of $\cN$, that agree
on an assignment on a prefix $\cS'(\cN)$ of $\cS(\cN)$, agree on ordering of
all the variables in $\cN$, the relative importance of which is fully
determined by $\cS'(\cN)$.
}

\subsection{Dominance Testing for TCP-nets}
\label{ss:tcp-dominance}
One of the most fundamental queries in any preference-representation
formalism is whether some outcome $o$ dominates (i.e., is strictly
preferred to) some other outcome $o'$. As discussed above, such {\em dominance
queries} are required whenever we wish to generate more than one
non-dominated solution to a set of hard constrains.  Much like in CP-nets, a dominance query
$\langle \cN, o, o'\rangle$ with respect to a TCP-net can be treated
as a search for an improving flipping sequence from the (purported)
less preferred outcome $o'$ to the (purported) more preferred outcome
$o$ through a sequence of successively more preferred outcomes, such
that each flip in this sequence is directly sanctioned by the given
TCP-net. Formally, an improving flipping sequence in the context of
TCP-nets can be defined as follows:

\begin{definition}
  \label{def:tcp-flip}
  A sequence of outcomes
  \[
  o' = o_0 \prec o_1 \prec \cdots \prec o_{m-1} \prec o_m = o
  \]
  is an {\em improving flipping sequence with respect to a \TCPn}\
  $\cN$ if and only if, for $0 \leq i < m$, either
  \begin{enumerate}
  \item {\it (CP-flips)} outcome $o_i$ is different from the outcome
    $o_{i+1}$ in the value of exactly one variable $X_j$, and $o_i[j]
    \prec o_{i+1}[j]$ given the (identical) values of $Pa(X_j)$ in $o_i$
    and $o_{i+1}$, or
  \item {\it (I-flips)} outcome $o_i$ is different from the outcome
    $o_{i+1}$ in the value of exactly {\em two} variables $X_j$ and
    $X_k$, $o_i[j] \prec o_{i+1}[j]$ and $o_i[k] \succ o_{i+1}[k]$
    given the (identical) values of $Pa(X_j)$ and $Pa(X_k)$ in $o_i$ and
    $o_{i+1}$, and $X_j \imp X_k$ given $\cR\cI(X_j,X_k|\bZ)$ and the
    (identical) values of $\bZ$ in $o_i$ and $o_{i+1}$.%
    \footnote{We implicitly assumed that neither node is the
    parent of the other. An implicit consequence of the standard semantics
    of conditional preferences is a node is more important than
    its children. Thus, there is no need to specify this
    explicitly.}
  \end{enumerate}
\end{definition}

Clearly, each value flip in such a flipping sequence is sanctioned by
the \TCPn\ $\cN$, and the CP-flips are exactly the flips allowed in
CP-nets~\cite{BBDHP.journal}.

\begin{theorem}
  \label{t:tcp-flip}
  Given a TCP-net $\cN$ and a pair of outcomes $o$ and $o'$, we have that
  $\cN \models o \succ o'$ if and only if there is an improving flipping
  sequence with respect to $\cN$ from $o'$ to $o$.
\end{theorem}

\begin{proof}
  \label{page-t:tcp-flip}

  \vspace{0.1in}
  \noindent $\Longleftarrow$  Given an improving flipping sequence $\cF$:
 \[
  o' = o_0 \prec o_1 \prec \cdots \prec o_{m-1} \prec o_m = o
  \]
  from $o'$ to $o$ with
  respect to $\cN$,
  by Definition~\ref{def:tcp-flip},  we have
  $\cN \models o_i
  \succ o_{i+1}$ for any improving flip from $\cF$.
  The  proposition follows from the transitivity of preferential entailment
  with respect to TCP-nets (Lemma~\ref{l:tcp-transitive}).

  \vspace{0.2in}
  \noindent $\Longrightarrow$ Let $\cG$ be the graph of preferential ordering
  induced by $\cN$, i.e., nodes of $\cG$ stand for all outcomes, and
  there is a directed edge from $o_1$ to $o_2$ if and only if there is
  an improving CP-flip or I-flip of $o_1$ to $o_2$, sanctioned by
  $\cN$.  Clearly, directed paths in $\cG$ are equivalent to improving
  flipping sequences with respect to $\cN$.

  First, we show that any preference ordering $\succ$ that respects
  the paths in $\cG$ (that is, if there is a path from $o_1$ to $o_2$
  in $\cG$, then we have $o_2 \succ o_1$) satisfies $\cN$.  Assume to
  the contrary that $\succ^*$ respects the paths in $\cG$, and does
  not satisfy $\cN$. Then, by the definition of satisfiability
  (Definition~\ref{def:tcp-satisf2}), there must exist either:
  \begin{enumerate}
  \item Some variable $X$, assignment $\bp\in\cD(Pa(X))$, values
    $x,x'\in\cD(X)$, and assignment $\bw$ to the remaining variables $\bW
    = \bV - (X\cup Pa(X))$, such that $\bp x \bw \succ^* \bp x' \bw$,
    but $CPT(X)$ dictates that $x' \succ x$ given $\bp$, or
  \item Some importance arc $\xi$ between a pair of variables $X$ and $Y$,
    assignment $\bz\in \cD(\cS(\xi))$ (if $\xi$ is an
    ${\mathsf{i}}$-arc, then $\cS(\xi) = \emptyset$), values
    $x,x'\in\cD(X), y,y'\in\cD(Y)$, and assignment $\bw$ to the
    remaining variables $\bW = \bV - (\{X,Y\}\cup \cS(\xi))$, such
    that $\bp x y \bw \succ^* \bp x' y' \bw$, but (i) the $CPT(X)$ dictates
    that $x' \succ x$, and (ii) the (possibly empty) $CIT$ of $\xi$
    dictates that $X \imp Y$ given $\bz$.
  \end{enumerate}
  However, in the first case, if $\cN$ specifies $x' \succ x$ given
  $\bp$, there is a CP-flip from $\bp x' \bw$ to $\bp x \bw$,
  contradicting the fact that $\succ^*$ extends $\cG$. Similarly, in
  the second case, if $\cN$ specify $x' \succ x$ given $\bw$, and $X
  \imp Y$ given $z$, then there is an I-flip from $\bp x' y' \bw$ to
  $\bp x y \bw$, contradicting the fact that $\succ^*$ extends $\cG$.

  Now, by the construction of $\cG$, if there is no improving flipping
  sequence from $o'$ to $o$, then there is no directed path in $\cG$
  from $o'$ to $o$. Therefore, there exist a preference ordering  $\succ^*$
  respecting the paths in $\cG$  in which $o' \succ^* o$.
  However, based on the above observation on preference orderings
  respecting the paths in $\cG$, $\succ^*$ also satisfies $\cN$, which implies
  $\cN \not\models o \succ o'$.
\end{proof}


Various
methods can be used to search for a flipping sequence. In
particular, we believe that at least some of the techniques,
developed for this task with respect to CP-nets
by~\citeA{Domshlak:Brafman:kr02,Domshlak:PhD}, and~\citeA{BBDHP.journal} can be
applied to the \TCPn\ model -- an issue left for future
research.  However, in general, dominance testing with respect to
CP-nets (and thus TCP-nets) is known to be {\sc
NP}-hard~\cite{BBDHP.journal}, thus in practice one may possibly
consider performing approximate constrained optimization, using
the \searchtcp\ algorithm with a dominance  testing based on one
of the tractable refinements of TCP-nets such as those discussed
by~\citeA{Brafman:Domshlak:uai04}.

\section{Discussion}
\label{S:summary}
CP-nets~\cite{BBHP.UAI99,BBDHP.journal}
is a relatively new graphical model for representation and reasoning about preferences.
Its development, however, already stimulated research in several  directions
(e.g., see~\shortcite{Brafman:Chernyavsky,Brafman:Dimopoulos:ci04,Brewka:aaai02,Boutilier:Bacchus:Brafman:uai01,DRVW:ijcai03,RVW:aaai04,Lang:kr02,Wilson:aaai04,Wilson:ecai04}).
In this paper we introduced the qualitative notions of absolute and
conditional relative importance between pairs of variables and
extended the CP-net model to capture the corresponding preference
statements. The extended model is called \TCPns. We identified a wide
class of \TCPns\ that are satisfiable, notably the class of
conditionally acyclic \TCPns, and analyzed complexity and algorithms
for testing membership in this class of networks. We also studied
reasoning about \TCPns, focusing on outcome optimization in
conditionally acyclic \TCPns\ with and without hard constraints.

Our work opens several directions for future
research.  First, an important open theoretical question is the
precise complexity of dominance testing in \TCPns.  In the context of
CP-nets this problem has been studied
by~\shortciteA{Domshlak:PhD,BBDHP.journal,Goldsmith:etal:ijcai05}.
Another question is the consistency of \TCPns\ that are not
conditionally acyclic. A preliminary study of this issue in context
of cyclic CP-nets has been done by~\citeA{Domshlak:Brafman:kr02} and~\shortciteA{Goldsmith:etal:ijcai05}.

The growing research on preference modeling
is motivated by the need for preference elicitation,
representation, and reasoning techniques in diverse areas of AI and user-centric information systems. In particular, one
of the main application areas we have in mind is this of automatic personalized product configuration~\cite{Sabin:Weigel:ieee98}. Thus,
in the remaining part of this section, we first consider the process of preference elicitation with TCP-nets, listing a few practical challenges that should be addressed to make this process appealing to
users en-masse. Then, we relate our work to some
other approaches to preference-based optimization.

\subsection{Preference Elicitation with TCP-nets (and Other Logical Models of Preference)}

The process of preference elicitation is known to be complex
as into account should be taken not only the formal model of the user's preferences but also numerous important factors of human-computer interaction (e.g.,
see~\cite{Pu:Faltings:etal:iui04,Pu:Faltings:constraints04}). In this
paper we focus on a formalism for structuring and analyzing the
user's preferences, although for some (probably offline)
applications, this formalism could actually be used to drive the
input process, much like a Bayes network can be used to help experts
express their beliefs.

Depending on the application, a {\em schematic} process of
constructing a \TCPn\ would commence by asking the decision maker
to identify the variables of interest, or by presenting them to
the user, if they are fixed. For example, in the application of
CP-net to adaptive document
presentation~\cite{Domshlak:Brafman:Shimony:ijcai01,Domshlak:Brafman:Shimony:tois04},
the content provider chooses a set of content elements, which
correspond to the set of variables. For an online
shopper-assistant agent, the variables are likely to be fixed
(e.g., if the agent is an online PC
customizer)~\cite{Brafman:Domshlak:uai04}. Next, the user is
asked to consider for each variable, the value of which other
variables influence her preferences over the values of this
variable. At this point ${\mathsf{cp}}$-arcs and CPTs are
introduced. Next, the user is asked to consider relative
importance relations, and the ${\mathsf{i}}$ and
${\mathsf{ci}}$-arcs are added. For each ${\mathsf{ci}}$-arc, the
corresponding CIT is filled.

Clearly, one may prefer to keep the preference elicitation process more user-driven, allowing the user simply provide us with a set of preference statements. But if such a set of statements fits the language expressible by the TCP-nets model, then the specific TCP-net underlying these statements can be constructed from a simple analysis of referents and conditionals of these statements. Such TCP-net extraction from the statements will be simpler if these statements will be provided in this or another formal language, or obtained via some carefully designed, structured user interface. However, for the user it is obviously more natural to provide these statements in natural language. Hence,
an interesting practical question related to elicitation of
qualitative preferences is model acquisition from
speech and/or text~\cite{Asher:Morreau:95,james99challenges,Bethard:etal:04}. Observe that
the intuitiveness of the qualitative preferential statements is
closely related to the fact that they have a straightforward
representation in natural language of everyday life.  In addition,
collections of typical preferential statements seem to form a linguistic domain that is a priori constrained in a very special manner.  This may
allow us to develop specialized techniques and tools for understanding
the corresponding language. Both offline and online language
understanding should be considered, since a user can either describe
her preferences offline, as a self-contained text, or can be asked
online, as a part of interactive process of (possibly mixed)
preference elicitation and preference-based constrained
optimization.

Yet another possible approach for eliciting  TCP-nets, as well as some alternative logical models of preferences, would be to allow the user expressing pair-wise comparisons between completely specified choices, and then construct a TCP-net consistent with this input. In the scope of quantitative models for preference representation, such an example-based model generation has been adopted in numerous user-centric optimization systems (e.g., see~\cite{Linden:Hanks:Lesh:97,veil}.) However, devising such a framework for learning qualitative models of preference seems to be somewhat more challenging. In theory, nothing prevents us from adopting example-based generation of TCP-nets since the latter can be seen as just a compact representation of a preference relation over a space of choices. The question, however, is whether a reasonably small set of pair-wise choice comparisons can provide us with a sufficient basis for learning not just a TCP-net consistent with these ``training examples'', but a compact TCP-net that will generalize  in a justifiable manner beyond the provided examples. To the best of our knowledge, so far this question has been studied for no logical preference-representation models, and hence it clearly poses a challenging venue for future research.\footnote{Note that, if we are only interested in compact modeling of pair-wise comparisons between the choices, then numerous techniques from the area of machine learning can be found useful. For instance, one can learn a decision tree classifying {\em ordered pairs} of choices as ``preferred'' (first choice to the second choice) or ``not preferred''. However, such a classification does not guarantee in general that the resulting binary relation over the space of choices will be anti-symmetric under the assumption of preference transitivity (a joint property that considered to be extremely natural in the literature on preference structures~\cite{Hansson:01}.)}

\subsection{Related Work}

As we show in Section~\ref{S:constraints}, extending CP-nets to TCP-nets is
appealing mainly in the scope of decision scenarios where the space of all
syntactically possible choices is (either explicitly or implicitly) constrained by some hard constraints.
We now review
some related approaches to preference-based optimization that
appeared in the literature.

A primary example of preference-based
optimization is the soft-constraints formalism (e.g., see~\shortciteA{Semiring:acm97}),
developed to model constraint satisfaction problems that are
either over-constrained (and thus unsolvable according to the
standard meaning of satisfaction)~\cite{Freuder:Wallace:aij92},
or suffer from imprecise knowledge about the actual
constraints~\cite{Fargier:Lang:93}.  In this formalism, the
constrained optimization problem is represented as a set of
preference orders over assignments to subsets of variables,
together with some operator for combining the preference relations
over these subsets of variables to a preference relation over the
assignments to the whole set of variables.  Each such subset of
variables corresponds to a soft constraint that can be
satisfied to different extent by different variable assignments.
There is much flexibility in how such ``local'' preference orders
are specified, and how they are combined. Various
soft-constraints models, such as weighted~\shortcite{softCSP:c99},
fuzzy~\cite{Schiex:uai92}, probabilistic~\cite{Fargier:Lang:93},
and lexicographic~\shortcite{Fargier:Lang:Schiex:eufit93} CSPs, are
discussed in the literature on soft constraint satisfaction.

The conceptual difference between our approach and the
soft-constraints formalism is that the latter is based on tightly
coupled representation of preferences and constraints, while our
representation of these two paradigms is completely decoupled.
Informally, soft constraints machinery has been developed for
optimization {\em of} partial constraint satisfaction, while we
are dealing with optimization {\em in the face} of constraint
satisfaction.  For instance, in personalized product
configuration, there are two parties involved typically: the
manufacturer and the consumer.  The manufacturer brings forth its
product expertise, and with it a set of hard constraints on
possible system configurations and the operating environment.
The user expresses her preferences over properties of the final
product. Typically numerous configurations satisfy the production
constraints, and the manufacturer strives to provide the user
with maximal satisfaction by finding one of the most preferred,
feasible product configuration. This naturally leads to a
decoupled approach.

\citeA{Freuder:OSullivan:cp01} proposed a framework of
interactive sessions for over-constrained problems. During such a session, if
the constraint solver discovers that no solution for the current set of
constraints is available, the user is asked to consider ``tradeoffs''.  For
example, following~\citeA{Freuder:OSullivan:cp01}, suppose that the set of
user requirements for a photo-camera configuration tool is that the weight of
the camera should be less that 10 ounces and the zoom lens should be at least
10X. If no camera meets these requirements, the user may specify tradeoffs such
as ``I will increase my weight limit to 14 ounces, if I can have a zoom lens of
at least 10X'' (possibly using some suggestions automatically generated by the
tool). In turn, these tradeoffs are used for refining the current set of
requirements, and the system goes into a new constraint satisfaction process.

The tradeoffs exploited by~\citeA{Freuder:OSullivan:cp01} correspond to the
information captured in TCP-nets by the ${\mathsf{i}}$-arcs.  However, instead
of treating this information as an incremental ``compromising'' update to the
set of hard constraints as done by~\citeA{Freuder:OSullivan:cp01}, in the TCP-net
based constrained optimization presented in Section~\ref{S:constraints}, we
exploit this information to guide the constraint solver to the preferable
feasible solutions. On the other hand,
the motivation and ideas behind the
work of Freuder and O'Sullivan as well as these in some other works on
interactive search (see works, e.g., on interactive goal
programming~\cite{dyer:mgmtsci72}, and interactive optimization based on example critique~\cite{Pu:Faltings:constraints04}) open a venue for future research on
interactive preference-based constrained optimization with TCP-nets, where elicitation of the user preferences is to be interleaved with the search for
feasible solution.

The notion of lexicographic orders/preferences~\shortcite{Fishburn:74,Schiex:Fargier:Verfaillie:95,Freuder:ordinalCSP} is closely
related to our notion of importance. The idea of a lexicographic
ordering is often used in qualitative approaches for
multi-criteria decision making. Basically, it implies that if one
item does better than another on the most important
(lexicographically earlier) criteria on which they differ, it is
considered better overall, regardless on how poorly it may do on
all other criteria. Thus, if we have four criteria (or attributes)
$A,B,C,D$, thus ordered, and $o$ does well on $A$ but miserably on
$B,C$ and $D$, whereas $o'$ is slightly worse on $A$ but much
better on all other criteria, $o$ is still deemed better. In terms of
our notion of variable importance, a lexicographic ordering of
attributes denotes a special form of relative importance of an
attribute versus a set of attributes. Thus, in the example above,
$A$ is more important than $B,C$ and $D$ combined; $B$ is more
important than $C$ and $D$ combined, and $C$ is more important
than $D$. We note that~\citeA{Wilson:aaai04} provides a nice language
that can capture such statements and more. Wilson allows
statements of the form ``$A=a$ is preferred to $A=a$
all-else-being-equal, except for $B$ and $C$.'' That is, given two
outcomes that differ on $A,B$ and $C$ only, the one that assigns
$a$ to $A$ is preferred to the one that assigns $a'$, regardless of the value of $B$
and $C$ in these outcomes. 
Hence, this richer language can in particular capture
lexicographic preferences.

While we believe a lexicographic ordering over attributes is
typically too strong, we think the flexibility provided by
Wilson's language could be quite useful. However, once one starts
analyzing relationships between sets of attributes, the utility
of graphical models and their analysis power becomes
questionable. Indeed, we are not aware of a graphical analysis of
Wilson's approach, except for the special case covered by \TCPns.
Moreover, our intuition is that relative importance of sets will
be a notion that users are much less comfortable specifying in
many applications. However, this hypothesis requires
empirical verification, as well as a more general study of the exact
expressive power of \TCPns, i.e., characterizing partial orders
that are expressible using this language. We believe that this is an
important avenue for future research.



\subsection*{Acknowledgments} 
Ronen Brafman and Solomon Shimony were partly supported by 
the Paul Ivanier Center for Robotics and Production Management.
Ronen Brafman was partly supported by NSF grants SES-0527650 and
IIS-0534662. Ronen Brafman's permanent address is: Department of Computer Science, Ben Gurion University, Israel.


\bibliographystyle{theapa}
\bibliography{brafman06a}

\begin{thebibliography}{}

\bibitem[\protect\BCAY{Asher\ \BBA\ Morreau}{Asher\ \BBA\
  Morreau}{1995}]{Asher:Morreau:95}
Asher, N.\BBACOMMA\  \BBA\ Morreau, M. \BBOP1995\BBCP.
\newblock \BBOQ What some generic sentences mean\BBCQ\
\newblock In Carlson, G.\BBACOMMA\  \BBA\ Pelletier, F.~J.\BEDS, {\Bem The
  Generic Book}, \BPGS\ 300--338. Chicago University Press.

\bibitem[\protect\BCAY{Bethard, Yu, Thornton, Hatzivassiloglou,\ \BBA\
  Jurafsky}{Bethard et~al.}{2004}]{Bethard:etal:04}
Bethard, S., Yu, H., Thornton, A., Hatzivassiloglou, V., \BBA\ Jurafsky, D.
  \BBOP2004\BBCP.
\newblock \BBOQ Automatic extraction of opinion propositions and their
  holders\BBCQ\
\newblock In {\Bem Proceedings of the AAAI Spring Symposium on Exploring
  Attitude and Affect in Text: Theories and Applications}.

\bibitem[\protect\BCAY{Bistarelli, Fargier, Montanari, Rossi, Schiex,\ \BBA\
  Verfaillie}{Bistarelli et~al.}{1999}]{softCSP:c99}
Bistarelli, S., Fargier, H., Montanari, U., Rossi, F., Schiex, T., \BBA\
  Verfaillie, G. \BBOP1999\BBCP.
\newblock \BBOQ Semiring-based {CSP}s and valued {CSP}s: Frameworks,
  properties, and comparison\BBCQ\
\newblock {\Bem Constraints}, {\Bem 4\/}(3), 275--316.

\bibitem[\protect\BCAY{Bistarelli, Montanari,\ \BBA\ Rossi}{Bistarelli
  et~al.}{1997}]{Semiring:acm97}
Bistarelli, S., Montanari, U., \BBA\ Rossi, F. \BBOP1997\BBCP.
\newblock \BBOQ Semiring-based constraint solving and optimization\BBCQ\
\newblock {\Bem Journal of the ACM}, {\Bem 44\/}(2), 201--236.

\bibitem[\protect\BCAY{Blythe}{Blythe}{2002}]{veil}
Blythe, J. \BBOP2002\BBCP.
\newblock \BBOQ Visual exploration and incremental utility elicitation\BBCQ\
\newblock In {\Bem Proceedings of the National Conference on Artificial
  Intelligence (AAAI)}, \BPGS\ 526--532.

\bibitem[\protect\BCAY{Boutilier, Bacchus,\ \BBA\ Brafman}{Boutilier
  et~al.}{2001}]{Boutilier:Bacchus:Brafman:uai01}
Boutilier, C., Bacchus, F., \BBA\ Brafman, R.~I. \BBOP2001\BBCP.
\newblock \BBOQ {UCP}-networks: {A} directed graphical representation of
  conditional utilities\BBCQ\
\newblock In {\Bem Proceedings of Seventeenth Conference on Uncertainty in
  Artificial Intelligence}, \BPGS\ 56--64.

\bibitem[\protect\BCAY{Boutilier, Brafman, Domshlak, Hoos,\ \BBA\
  Poole}{Boutilier et~al.}{2004a}]{BBDHP.journal}
Boutilier, C., Brafman, R., Domshlak, C., Hoos, H., \BBA\ Poole, D.
  \BBOP2004a\BBCP.
\newblock \BBOQ {CP}-nets: {A} tool for representing and reasoning about
  conditional {\em ceteris paribus} preference statements\BBCQ\
\newblock {\Bem Journal of Artificial Intelligence Research (JAIR)}, {\Bem 21},
  135--191.

\bibitem[\protect\BCAY{Boutilier, Brafman, Domshlak, Hoos,\ \BBA\
  Poole}{Boutilier et~al.}{2004b}]{BBDHP.journal2}
Boutilier, C., Brafman, R., Domshlak, C., Hoos, H., \BBA\ Poole, D.
  \BBOP2004b\BBCP.
\newblock \BBOQ Preference-based constrained optimization with {CP}-nets\BBCQ\
\newblock {\Bem Computational Intelligence (Special Issue on Preferences in AI
  and CP)}, {\Bem 20\/}(2), 137--157.

\bibitem[\protect\BCAY{Boutilier, Brafman, Hoos,\ \BBA\ Poole}{Boutilier
  et~al.}{1999}]{BBHP.UAI99}
Boutilier, C., Brafman, R., Hoos, H., \BBA\ Poole, D. \BBOP1999\BBCP.
\newblock \BBOQ Reasoning with conditional {\em ceteris paribus} preference
  statements\BBCQ\
\newblock In {\Bem Proceedings of the Fifteenth Annual Conference on
  Uncertainty in Artificial Intelligence}, \BPGS\ 71--80. Morgan Kaufmann
  Publishers.

\bibitem[\protect\BCAY{Brafman\ \BBA\ Chernyavsky}{Brafman\ \BBA\
  Chernyavsky}{2005}]{Brafman:Chernyavsky}
Brafman, R.\BBACOMMA\  \BBA\ Chernyavsky, Y. \BBOP2005\BBCP.
\newblock \BBOQ Planning with goal preferences and constraints\BBCQ\
\newblock In {\Bem Proceedings of the International Conference on Automated
  Planning and Scheduling}, \BPGS\ 182--191, Monterey, CA.

\bibitem[\protect\BCAY{Brafman\ \BBA\ Domshlak}{Brafman\ \BBA\
  Domshlak}{2002}]{Brafman:Domshlak:uai02}
Brafman, R.\BBACOMMA\  \BBA\ Domshlak, C. \BBOP2002\BBCP.
\newblock \BBOQ Introducing variable importance tradeoffs into {CP}-nets\BBCQ\
\newblock In {\Bem Proceedings of the Eighteenth Annual Conference on
  Uncertainty in Artificial Intelligence}, \BPGS\ 69--76, Edmonton, Canada.

\bibitem[\protect\BCAY{Brafman, Domshlak,\ \BBA\ Kogan}{Brafman
  et~al.}{2004a}]{Brafman:Domshlak:uai04}
Brafman, R., Domshlak, C., \BBA\ Kogan, T. \BBOP2004a\BBCP.
\newblock \BBOQ Compact value-function representations for qualitative
  preferences\BBCQ\
\newblock In {\Bem Proceedings of the Twentieth Annual Conference on
  Uncertainty in Artificial Intelligence}, \BPGS\ 51--58, Banff, Canada.

\bibitem[\protect\BCAY{Brafman, Domshlak,\ \BBA\ Shimony}{Brafman
  et~al.}{2004b}]{Domshlak:Brafman:Shimony:tois04}
Brafman, R., Domshlak, C., \BBA\ Shimony, S.~E. \BBOP2004b\BBCP.
\newblock \BBOQ Qualitative decision making in adaptive presentation of
  structured information\BBCQ\
\newblock {\Bem ACM Transactions on Information Systems}, {\Bem 22\/}(4),
  503--539.

\bibitem[\protect\BCAY{Brafman\ \BBA\ Dimopoulos}{Brafman\ \BBA\
  Dimopoulos}{2004}]{Brafman:Dimopoulos:ci04}
Brafman, R.~I.\BBACOMMA\  \BBA\ Dimopoulos, Y. \BBOP2004\BBCP.
\newblock \BBOQ Extended semantics and optimization algorithms for
  cp-networks\BBCQ\
\newblock {\Bem Computational Intelligence (Special Issue on Preferences in AI
  and CP)}, {\Bem 20\/}(2), 218--245.

\bibitem[\protect\BCAY{Brafman\ \BBA\ Friedman}{Brafman\ \BBA\
  Friedman}{2005}]{BF05}
Brafman, R.~I.\BBACOMMA\  \BBA\ Friedman, D. \BBOP2005\BBCP.
\newblock \BBOQ Adaptive rich media presentations via preference-based
  constrained optimization\BBCQ\
\newblock In {\Bem Proceedings of the IJCAI-05 Workshop on Advances in
  Preference Handling}, \BPGS\ 19--24, Edinburgh, Scotland.

\bibitem[\protect\BCAY{Brewka}{Brewka}{2002}]{Brewka:aaai02}
Brewka, G. \BBOP2002\BBCP.
\newblock \BBOQ Logic programming with ordered disjunction\BBCQ\
\newblock In {\Bem Proceedings of Eighteenth National Conference on Artificial
  Intelligence}, \BPGS\ 100--105, Edmonton, Canada. AAAI Press.

\bibitem[\protect\BCAY{Burke}{Burke}{2000}]{Burke:00}
Burke, R. \BBOP2000\BBCP.
\newblock \BBOQ Knowledge-based recommender systems\BBCQ\
\newblock In Kent, A.\BED, {\Bem Encyclopedia of Library and Information
  Systems}, \lowercase{\BVOL}~69, \BPGS\ 180--200. Marcel Dekker, New York.

\bibitem[\protect\BCAY{Domshlak}{Domshlak}{2002}]{Domshlak:PhD}
Domshlak, C. \BBOP2002\BBCP.
\newblock {\Bem Modeling and Reasoning about Preferences with {CP}-nets}.
\newblock Ph.D.\ thesis, Ben-Gurion University, Israel.

\bibitem[\protect\BCAY{Domshlak\ \BBA\ Brafman}{Domshlak\ \BBA\
  Brafman}{2002}]{Domshlak:Brafman:kr02}
Domshlak, C.\BBACOMMA\  \BBA\ Brafman, R. \BBOP2002\BBCP.
\newblock \BBOQ {CP}-nets - reasoning and consistency testing\BBCQ\
\newblock In {\Bem Proceedings of the Eighth International Conference on
  Principles of Knowledge Representation and Reasoning}, \BPGS\ 121--132,
  Toulouse, France.

\bibitem[\protect\BCAY{Domshlak, Brafman,\ \BBA\ Shimony}{Domshlak
  et~al.}{2001}]{Domshlak:Brafman:Shimony:ijcai01}
Domshlak, C., Brafman, R., \BBA\ Shimony, S.~E. \BBOP2001\BBCP.
\newblock \BBOQ Preference-based configuration of web page content\BBCQ\
\newblock In {\Bem Proceedings of the Seventeenth International Joint
  Conference on Artificial Intelligence}, \BPGS\ 1451--1456, Seattle.

\bibitem[\protect\BCAY{Domshlak, Rossi, Venable,\ \BBA\ Walsh}{Domshlak
  et~al.}{2003}]{DRVW:ijcai03}
Domshlak, C., Rossi, F., Venable, K.~B., \BBA\ Walsh, T. \BBOP2003\BBCP.
\newblock \BBOQ Reasoning about soft constraints and conditional preferences:
  Complexity results and approximation techniques\BBCQ\
\newblock In {\Bem Proceedings of the Eighteenth International Joint Conference
  on Artificial Intelligence}, \BPGS\ 215--220, Acapulco, Mexico.

\bibitem[\protect\BCAY{Dyer}{Dyer}{1972}]{dyer:mgmtsci72}
Dyer, J.~S. \BBOP1972\BBCP.
\newblock \BBOQ Interactive goal programming\BBCQ\
\newblock {\Bem Management Science}, {\Bem 19}, 62--70.

\bibitem[\protect\BCAY{Even, Itai,\ \BBA\ Shamir}{Even
  et~al.}{1976}]{Even:Itai:Shamir:76}
Even, S., Itai, A., \BBA\ Shamir, A. \BBOP1976\BBCP.
\newblock \BBOQ On the complexity of timetable and multicommodity flow
  problems\BBCQ\
\newblock {\Bem SIAM Journal on Computing}, {\Bem 5}, 691--703.

\bibitem[\protect\BCAY{Faltings, Pu, Torrens,\ \BBA\ Viappiani}{Faltings
  et~al.}{2004}]{Pu:Faltings:etal:iui04}
Faltings, B., Pu, P., Torrens, M., \BBA\ Viappiani, P. \BBOP2004\BBCP.
\newblock \BBOQ Designing example-critiquing interaction\BBCQ\
\newblock In {\Bem Proceedings of the International Conference on Intelligent
  User Interfaces}, \BPGS\ 22--29, Funchal, Madeira, Portugal.

\bibitem[\protect\BCAY{Fargier\ \BBA\ Lang}{Fargier\ \BBA\
  Lang}{1993}]{Fargier:Lang:93}
Fargier, H.\BBACOMMA\  \BBA\ Lang, J. \BBOP1993\BBCP.
\newblock \BBOQ Uncertainty in constraint satisfaction problems: A
  probabilistic approach\BBCQ\
\newblock In {\Bem Proceedings of the European Conference on Symbolic and
  Qualitative Approaches to Reasoning and Uncertainty}, \lowercase{\BVOL}\ 747
  of {\Bem LNCS}, \BPGS\ 97--104.

\bibitem[\protect\BCAY{Fargier, Lang,\ \BBA\ Schiex}{Fargier
  et~al.}{1993}]{Fargier:Lang:Schiex:eufit93}
Fargier, H., Lang, J., \BBA\ Schiex, T. \BBOP1993\BBCP.
\newblock \BBOQ Selecting preferred solutions in fuzzy constraint satisfaction
  problems\BBCQ\
\newblock In {\Bem Proceedings of the First European Congress on Fuzzy and
  Intelligent Technologies}, \BPGS\ 1128--1134.

\bibitem[\protect\BCAY{Fishburn}{Fishburn}{1974}]{Fishburn:74}
Fishburn, P. \BBOP1974\BBCP.
\newblock \BBOQ Lexicographic orders, utilities, and decision rules: A
  survey\BBCQ\
\newblock {\Bem Management Science}, {\Bem 20\/}(11), 1442--1471.

\bibitem[\protect\BCAY{French}{French}{1986}]{french}
French, S. \BBOP1986\BBCP.
\newblock {\Bem Decision Theory}.
\newblock Halsted Press, New York.

\bibitem[\protect\BCAY{Freuder\ \BBA\ O'Sullivan}{Freuder\ \BBA\
  O'Sullivan}{2001}]{Freuder:OSullivan:cp01}
Freuder, E.\BBACOMMA\  \BBA\ O'Sullivan, B. \BBOP2001\BBCP.
\newblock \BBOQ Generating tradeoffs for interactive constraint-based
  configuration\BBCQ\
\newblock In {\Bem Proceedings of the 7th International Conference on
  Principles and Practice of Constraint Programming}, \BPGS\ 590--594, Paphos,
  Cyprus.

\bibitem[\protect\BCAY{Freuder\ \BBA\ Wallace}{Freuder\ \BBA\
  Wallace}{1992}]{Freuder:Wallace:aij92}
Freuder, E.~C.\BBACOMMA\  \BBA\ Wallace, R.~J. \BBOP1992\BBCP.
\newblock \BBOQ Partial constraint satisfaction\BBCQ\
\newblock {\Bem Artificial Intelligence}, {\Bem 58}, 21--70.

\bibitem[\protect\BCAY{Freuder, Wallace,\ \BBA\ Heffernan}{Freuder
  et~al.}{2003}]{Freuder:ordinalCSP}
Freuder, E.~C., Wallace, R.~J., \BBA\ Heffernan, R. \BBOP2003\BBCP.
\newblock \BBOQ Ordinal constraint satisfaction\BBCQ\
\newblock In {\Bem Proceedings of the Fifth International Workshop on Soft
  Constraints}.

\bibitem[\protect\BCAY{Glass}{Glass}{1999}]{james99challenges}
Glass, J. \BBOP1999\BBCP.
\newblock \BBOQ Challenges for spoken dialogue systems\BBCQ\
\newblock In {\Bem Proceedings of the IEEE ASRU Workshop}, Keystone, CO.

\bibitem[\protect\BCAY{Goldsmith, Lang, Truszczynski,\ \BBA\ Wilson}{Goldsmith
  et~al.}{2005}]{Goldsmith:etal:ijcai05}
Goldsmith, J., Lang, J., Truszczynski, M., \BBA\ Wilson, N. \BBOP2005\BBCP.
\newblock \BBOQ The computational complexity of dominance and consistency in
  {CP}-nets\BBCQ\
\newblock In {\Bem Proceedings of the Nineteenth International Joint Conference
  on Artificial Intelligence}, \BPGS\ 144--149, Edinburgh, Scotland.

\bibitem[\protect\BCAY{Haag}{Haag}{1998}]{Haag:ieee98}
Haag, A. \BBOP1998\BBCP.
\newblock \BBOQ Sales configuration in business processes\BBCQ\
\newblock {\Bem IEEE Intelligent Systems and their Applications}, {\Bem
  13\/}(4), 78--85.

\bibitem[\protect\BCAY{Hansson}{Hansson}{2001}]{Hansson:01}
Hansson, S.~O. \BBOP2001\BBCP.
\newblock \BBOQ Preference logic\BBCQ\
\newblock In Gabbay, D.~M.\BBACOMMA\  \BBA\ Guenthner, F.\BEDS, {\Bem Handbook
  of Philosophical Logic\/} (2 \BEd)., \lowercase{\BVOL}~4, \BPGS\ 319--394.
  Kluwer.

\bibitem[\protect\BCAY{Keeney\ \BBA\ Raiffa}{Keeney\ \BBA\ Raiffa}{1976}]{KR}
Keeney, R.~L.\BBACOMMA\  \BBA\ Raiffa, H. \BBOP1976\BBCP.
\newblock {\Bem Decision with {M}ultiple {O}bjectives: {P}references and
  {V}alue {T}radeoffs}.
\newblock Wiley.

\bibitem[\protect\BCAY{Lang}{Lang}{2002}]{Lang:kr02}
Lang, J. \BBOP2002\BBCP.
\newblock \BBOQ From preference representation to combinatorial vote\BBCQ\
\newblock In {\Bem Proceedings of the Eight International Conference on
  Principles of Knowledge Representation and Reasoning (KR)}, \BPGS\ 277--288.

\bibitem[\protect\BCAY{Linden, Hanks,\ \BBA\ Lesh}{Linden
  et~al.}{1997}]{Linden:Hanks:Lesh:97}
Linden, G., Hanks, S., \BBA\ Lesh, N. \BBOP1997\BBCP.
\newblock \BBOQ Interactive assessment of user preference models: The automated
  travel assistant\BBCQ\
\newblock In {\Bem Proceedings of the Sixth International Conference on User
  Modeling}, \BPGS\ 67--78.

\bibitem[\protect\BCAY{Pu\ \BBA\ Faltings}{Pu\ \BBA\
  Faltings}{2004}]{Pu:Faltings:constraints04}
Pu, P.\BBACOMMA\  \BBA\ Faltings, B. \BBOP2004\BBCP.
\newblock \BBOQ Decision tradeoff using example critiquing and constraint
  programming\BBCQ\
\newblock {\Bem Constraints: An International Journal}, {\Bem 9\/}(4),
  289--310.

\bibitem[\protect\BCAY{Resnick\ \BBA\ Varian}{Resnick\ \BBA\
  Varian}{1997}]{acmc40}
Resnick, P.\BBACOMMA\  \BBA\ Varian, H.~R.\BEDS. \BBOP1997\BBCP.
\newblock {\Bem Special Issue on Recommender Systems}, \lowercase{\BVOL}~40 of
  {\Bem Communications of the ACM}.

\bibitem[\protect\BCAY{Rossi, Venable,\ \BBA\ Walsh}{Rossi
  et~al.}{2004}]{RVW:aaai04}
Rossi, F., Venable, K.~B., \BBA\ Walsh, T. \BBOP2004\BBCP.
\newblock \BBOQ m{CP} nets: Representing and reasoning with preferences of
  multiple agents\BBCQ\
\newblock In {\Bem Proceedings of the Nineteenth National Conference on
  Artificial Intelligence}, \BPGS\ 729--734, San Jose, CL.

\bibitem[\protect\BCAY{Sabin\ \BBA\ Weigel}{Sabin\ \BBA\
  Weigel}{1998}]{Sabin:Weigel:ieee98}
Sabin, D.\BBACOMMA\  \BBA\ Weigel, R. \BBOP1998\BBCP.
\newblock \BBOQ Product conguration frameworks - {A} survey\BBCQ\
\newblock {\Bem IEEE Intelligent Systems and their Applications}, {\Bem
  13\/}(4), 42--49.

\bibitem[\protect\BCAY{Schiex}{Schiex}{1992}]{Schiex:uai92}
Schiex, T. \BBOP1992\BBCP.
\newblock \BBOQ Possibilistic cosntraint satisfaction, or "{H}ow to handle soft
  constraints"\BBCQ\
\newblock In {\Bem Proceedings of Eighth Conference on Uncertainty in
  Artificial Intelligence}, \BPGS\ 269--275.

\bibitem[\protect\BCAY{Schiex, Fargier,\ \BBA\ Verfaillie}{Schiex
  et~al.}{1995}]{Schiex:Fargier:Verfaillie:95}
Schiex, T., Fargier, H., \BBA\ Verfaillie, G. \BBOP1995\BBCP.
\newblock \BBOQ Valued constraint satisfaction problems: Hard and easy
  problems\BBCQ\
\newblock In {\Bem Proceedings of the Fourteenth International Joint Conference
  on Artificial Intelligence}, \BPGS\ 631--637.

\bibitem[\protect\BCAY{Wilson}{Wilson}{2004a}]{Wilson:ecai04}
Wilson, N. \BBOP2004a\BBCP.
\newblock \BBOQ Consistency and constrained optimisation for conditional
  preferences\BBCQ\
\newblock In {\Bem Proceedings of the Sixteenth European Conference on
  Artificial Intelligence}, \BPGS\ 888--894, Valencia.

\bibitem[\protect\BCAY{Wilson}{Wilson}{2004b}]{Wilson:aaai04}
Wilson, N. \BBOP2004b\BBCP.
\newblock \BBOQ Extending {CP}-nets with stronger conditional preference
  statements\BBCQ\
\newblock In {\Bem Proceedings of the Nineteenth National Conference on
  Artificial Intelligence}, \BPGS\ 735--741, San Jose, CL.

\end{thebibliography}

\end{document}